\documentclass[twoside,11pt]{article}

\usepackage{blindtext}

\usepackage[utf8]{inputenc}
\usepackage[T1]{fontenc}
\usepackage{lmodern}
\usepackage{hyperref}
\usepackage{url}            
\usepackage{booktabs}       
\usepackage{amsfonts}       
\usepackage[numbers]{natbib}
\usepackage{tabularx}
\usepackage{nicefrac}       
\usepackage{algorithm,algpseudocode}
\usepackage{amsfonts}
\usepackage{multicol}
\usepackage{amsmath}
\usepackage{amsthm}
\usepackage{amssymb}
\usepackage[title]{appendix}
\usepackage{bm}
\usepackage{courier}
\usepackage[usenames,dvipsnames]{color}
\usepackage{enumitem}
\usepackage{graphicx}
\usepackage[margin=0.9in]{geometry}
\usepackage{url}
\usepackage{multirow}
\usepackage[dvipsnames]{xcolor}
\usepackage{epsfig}
\usepackage{graphicx}
\usepackage{subcaption}
\usepackage{wrapfig}   
\usepackage{caption}   

\hypersetup{
	colorlinks,
	linkcolor={red!80!black},
	citecolor={blue!50!black},
	urlcolor={blue!80!black}
}

\usepackage{subcaption}
\usepackage{graphicx}
\usepackage{caption}
\title{Stratified Sparse Attention: Reducing the Sparse-to-Dense Gap at Extreme Sparsity}

\newcommand{\samethanks}[1][\value{footnote}]{\footnotemark[#1]}

\date{}

\author{
\begin{tabular}{c@{\qquad}c@{\qquad}c}
Sahil Joshi\thanks{Equal contribution.}\,\,\thanks{\scriptsize Department of Computer Science, Rice University, TX, USA.} &
Prithvi Dixit\samethanks[1]\,\,\thanks{\scriptsize Department of Electrical Engineering and Computer Sciences, UC Berkeley, CA, USA.} &
Aryan Katiyar\thanks{\scriptsize Department of Computer Science, Stanford University, CA, USA.}
\\[8pt]
Agniva Chowdhury\samethanks[2] &
Joseph E. Gonzalez\samethanks[3] &
Anshumali Shrivastava\samethanks[2]
\\[8pt]
\multicolumn{3}{c}{
Aditya Desai\thanks{Co-lead.}\,\,\,\,\thanks{\scriptsize Department of Computer Science and Engineering, Indian Institute of Technology Bombay, India.}
\qquad
Kumar Krishna Agrawal\samethanks[5]\,\,\,\,\samethanks[3]
}
\end{tabular}
}
\usepackage{bibentry}

\input{./Definitions}

\begin{document}

\maketitle
\begin{abstract}
\noindent
Large Language Model inference over long contexts is increasingly bottlenecked by the cost of attending to a growing Key-Value (KV) cache. Sparse attention reduces this cost by accessing only a fraction of the cache, but existing methods face a trade-off between deterministic retrieval, which captures high-scoring tokens but discards residual attention mass, and sampling, which improves coverage at the cost of higher variance. Hybrid methods combine top-$k$ retrieval with uniform sampling, yet still exhibit a substantial gap from dense attention at aggressive sparsity. We introduce \emph{Stratified Sparse Attention (SSA)}, a unified framework combining deterministic retrieval with score-guided sampling. SSA deterministically attends to high-scoring tokens while sampling the residual using proxy attention scores and correcting for their inclusion probabilities, thereby recovering attention mass beyond the retrieval set. SSA encompasses top-$k$, top-$p$, and hybrid top-$k$+sampling as special cases. Across RULER, LongBench, LOFT, and AIME, and five models ranging from 3B to 27B parameters, SSA consistently improves over the baselines at matched sparsity. On RULER-HARD with Ministral-3-8B, SSA improves accuracy over the strongest practical baseline from 81.8\% to 84.2\% at 32K and from 63.8\% to 69.3\% at 128K under $50\times$ attention sparsity, recovering 94\% and 85\% of dense-attention performance, respectively. With optimized CUDA kernels, SSA achieves up to $2.96\times$ decode speedup over FlashInfer at 1M context, demonstrating that its accuracy gains translate into practical long-context inference acceleration. We open-source our code at \url{https://github.com/skylight-org/SSA-Paper-Code}.
\end{abstract}

\section{Introduction}

Large language models (LLMs) are rapidly expanding into long-context workloads, spanning large document collections~\citep{li2023large}, image analysis~\citep{dosovitskiy2021image, Peebles_2023_ICCV}, software engineering~\citep{hou2024large, chen2024deep}, databases~\citep{Zhou2024DBGPTLL}, video processing~\citep{ruan2021transformer}, and long-horizon agents that continuously accumulate observations, tool outputs, and interaction histories~\citep{cao2026coding, chen2026horizon}. As context lengths approach millions of tokens, efficient inference becomes increasingly challenging. During autoregressive decoding, each newly generated token attends to an ever-growing key-value (KV) cache, requiring keys and values to be repeatedly transferred from GPU high-bandwidth memory (HBM) onto on-chip SRAM for attention computation~\citep{shah2024flashattention, kwon2023efficient}. At long contexts, this makes decoding increasingly memory-bandwidth bound: the attention cost per decoding step grows linearly with context length. Sparse attention addresses this bottleneck by attending to only a small fraction of the KV cache~\citep{pmlr-v235-tang24l, liu2024retrievalattention, pqcache2025}, reducing both memory traffic and attention latency. However, realizing these gains in practice requires sparse attention to satisfy two competing objectives: \emph{faithfully approximate dense attention at aggressive sparsity} while remaining \emph{efficient enough to translate sparsity into end-to-end inference speedups}~\citep{pagliardini2023fast, lin2024s2attention}.
\\
\\
\noindent
Existing sparse-attention methods can be broadly categorized by how they allocate a limited attention budget. First, \emph{approximate retrieval} methods use proxy attention scores to deterministically select high-scoring tokens, typically through top-$k$ retrieval. Representative methods include PQCache~\citep{pqcache2025}, and Quest~\citep{pmlr-v235-tang24l}, with recent approaches such as Tactic~\citep{zhu2025topp} extending this paradigm to top-$p$ selection. While effective when attention is highly concentrated, deterministic retrieval can miss substantial attention mass when the distribution is diffuse or long-tailed~\citep{chen2024magicpig,desai2025vattention}. Second, \emph{sampling-based} methods such as MagicPig~\citep{chen2024magicpig} use importance sampling to approximate dense attention, providing broader coverage but introducing sampling variance that can lead to a larger gap from dense attention at limited budgets (see Table~\ref{tab:ruler_hard_32k_all_models} in Appendix). Finally, \emph{hybrid} methods such as vAttention~\citep{desai2025vattention} combine deterministic top-$k$ retrieval with uniform sampling~\citep{jeffreyuniform}, preserving high-confidence tokens while extending coverage beyond the top-$k$ set. Although this combination improves the trade-off between exploitation and coverage, a substantial gap to dense attention remains at aggressive sparsity (see Table~\ref{tab:ruler_hard_32k_all_models} in Appendix). See Section~\ref{sec:related} for more details. These observations motivate a fundamental question: \textit{Can we design a more effective sampling strategy that substantially narrows the sparse-to-dense accuracy gap?}

\begin{figure}[t]
    \centering

    \begin{subfigure}[t]{0.49\linewidth}
        \centering
        \includegraphics[width=\linewidth]{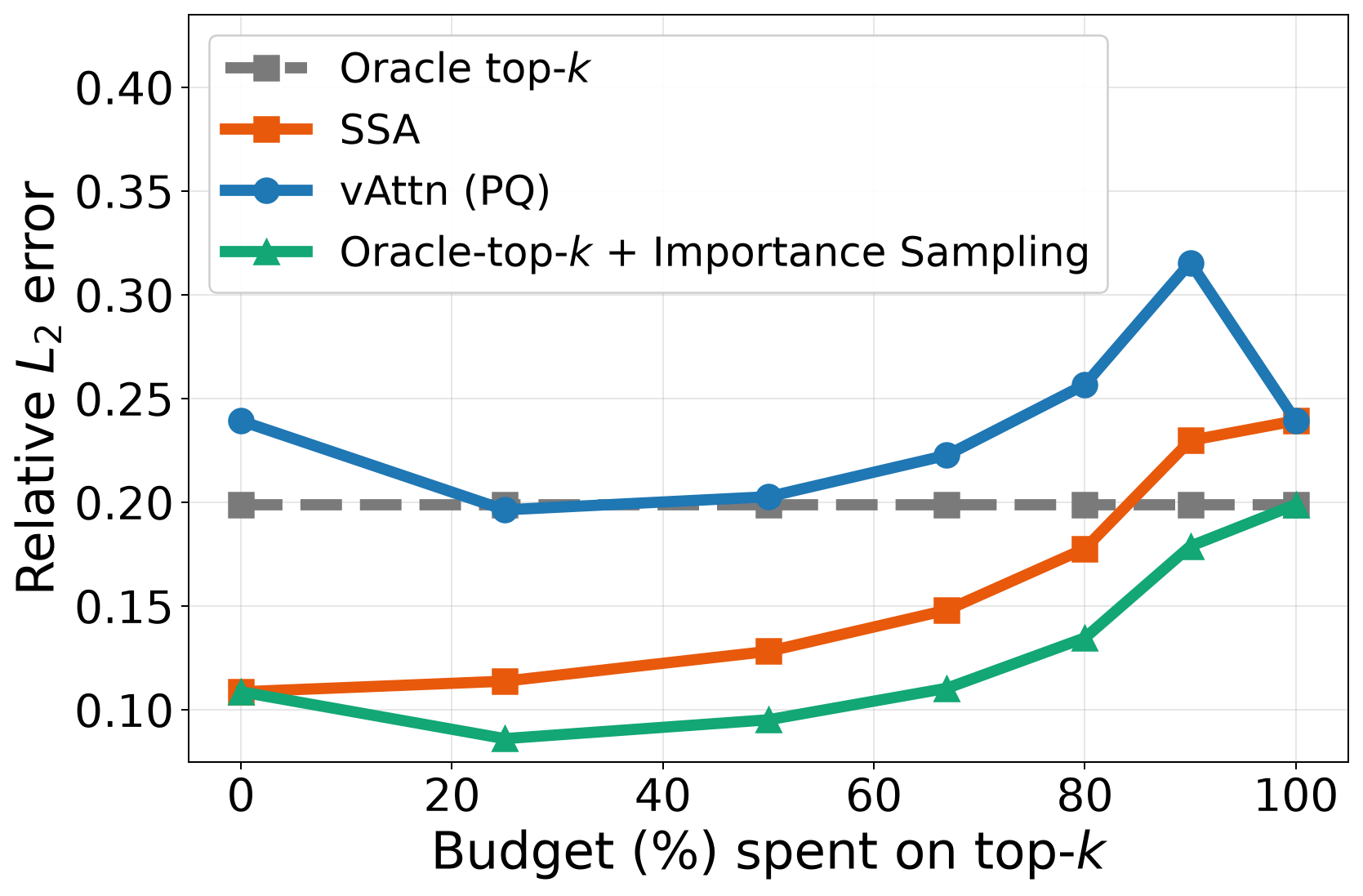}
    \end{subfigure}
    \hfill
    \begin{subfigure}[t]{0.49\linewidth}
        \centering
        \includegraphics[width=\linewidth]{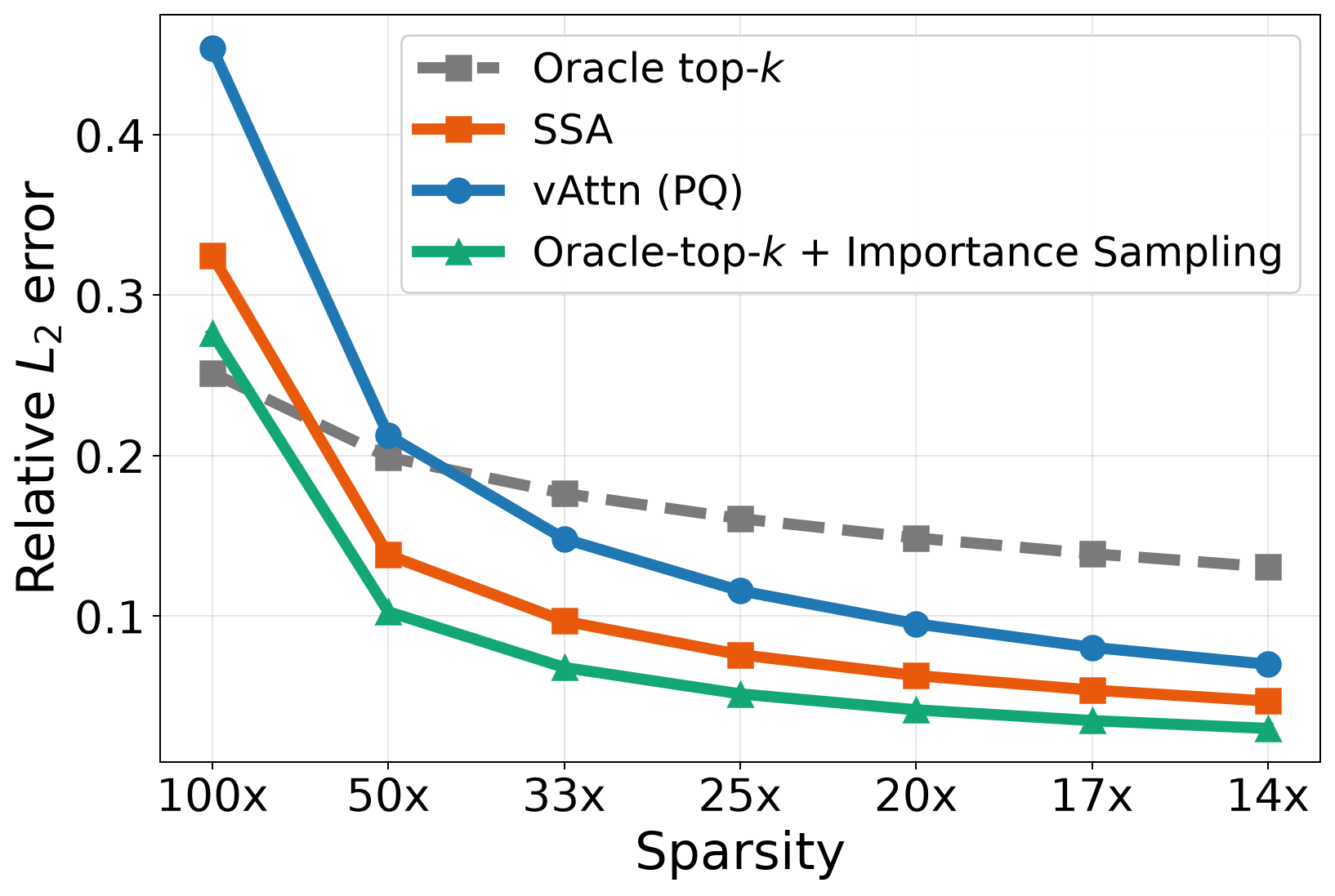}
    \end{subfigure}
    \caption{Relative $L_2$ error when varying the top-$k$:sampling budget allocation at $50\times$ sparsity (left) and varying sparsity with a fixed 3:2 top-$k$:sampling budget split (right).}
    \label{fig:ssa_error_analysis}
\end{figure}
\noindent
Our preliminary studies reveal a key limitation of uniform sampling: while it provides coverage beyond the deterministic top-$k$ set, it ignores the highly non-uniform distribution of residual attention mass\footnote{Oracle-top-\(k\) + Importance Sampling uses exact attention scores for both top-\(k\) selection and residual importance sampling, providing an oracle reference for SSA.}. As shown in fig.~\ref{fig:ssa_error_analysis}, score-guided importance sampling (the principle that SSA uses) consistently achieves lower relative $L_2$ error than uniform sampling. Motivated by this observation, we introduce \textbf{Stratified Sparse Attention (SSA)}, a unified framework that combines deterministic retrieval of high-confidence heavy hitters with \emph{score-guided stochastic exploration} of the remaining context. Rather than treating all residual tokens equally, SSA uses proxy attention scores to preferentially explore promising regions while preserving stochastic coverage of the long tail. This enables the sparse attention budget to better capture the attention mass that would otherwise be lost under aggressive sparsity. We evaluate SSA across diverse models, including Llama-3.1-8B-Instruct~\citep{llama31}, Qwen3.5-4B, Qwen3.5-9B, Ministral-3-3B-Instruct-2517, and Ministral-3-8B-Instruct-2517~\citep{liu2026ministral3}, and across RULER~\cite{hsieh2024ruler}, LongBench~\cite{bai2024longbench}, LOFT~\citep{lee2024longcontext}, and AIME~\cite{dekoninck2026matharena}. Across these evaluations, SSA consistently improves over deterministic, sampling-based, and hybrid sparse-attention baselines, substantially narrowing the gap to dense attention at extreme sparsity. Our contributions can be summarized as follows: 
\\
\noindent
\textbf{I. Unified Attention Paradigm:} We introduce a unified sparse-attention framework that partitions tokens into deterministic and stochastic strata using a generic proxy score. By varying how tokens are retained or sampled, the framework recovers top-$k$, top-$p$, and hybrid top-$k$+sampling as special cases of importance-sampling-inspired sparse attention. \
\\
\noindent
\textbf{II. Theoretical Guarantees:} We provide rigorous bounds for our proposed algorithm, characterizing how closely it approximates dense attention under a limited attention budget. \
\\
\noindent
\textbf{III. Efficient Kernel Support:} We develop efficient CUDA/Triton kernels that realize the computational benefits of our approach in practice, delivering end-to-end decoding speedups over dense attention.
\begin{figure*}[t]
    \centering

    \begin{minipage}[t]{0.48\textwidth}
        \centering
        \includegraphics[width=\linewidth]{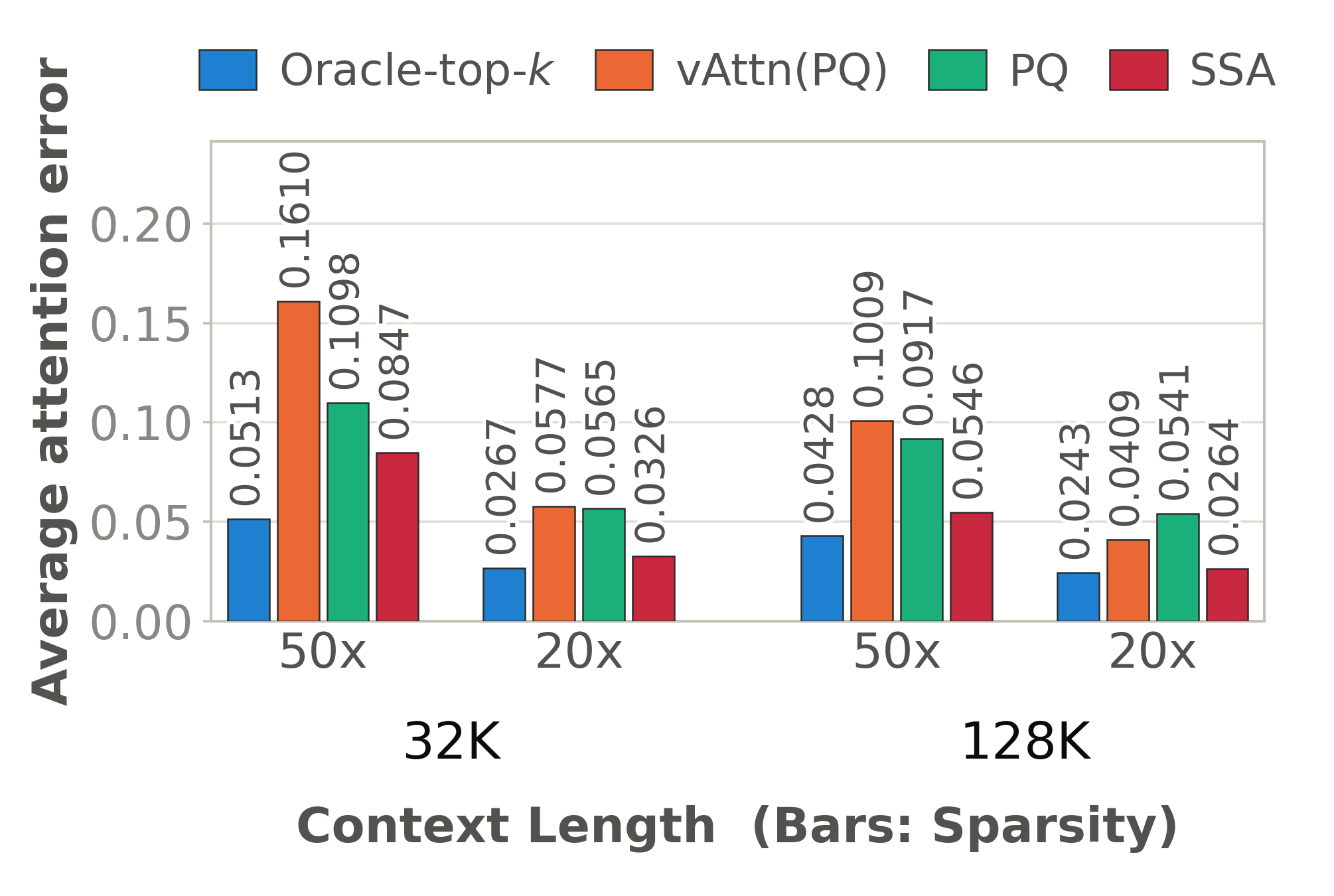}
        \vspace{-0.8em}
        
        {(a) Qwen3.5-4B: Error}
    \end{minipage}
    \hfill
    \begin{minipage}[t]{0.48\textwidth}
        \centering
        \includegraphics[width=\linewidth]{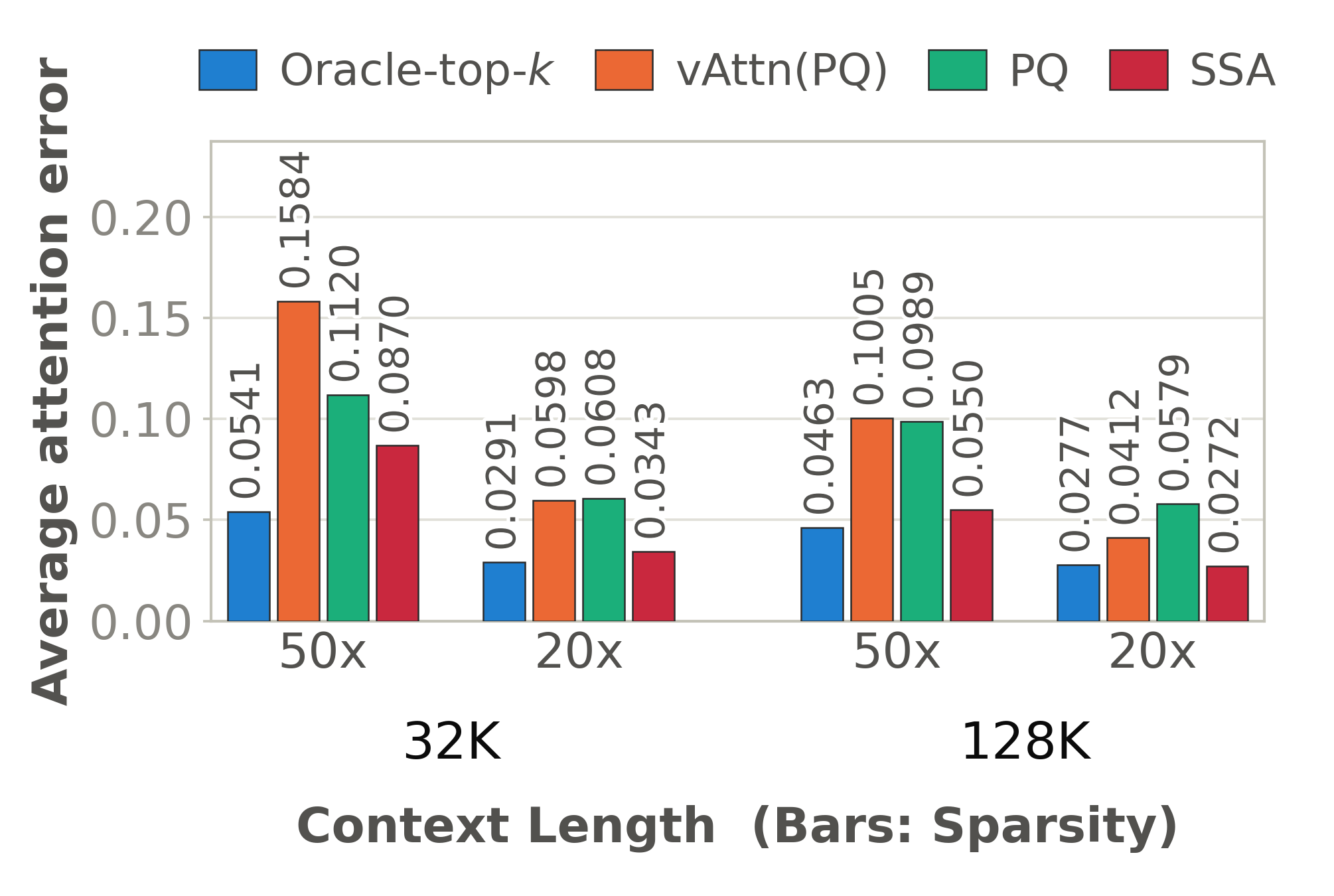}
        \vspace{-0.8em}
        
        {(b) Qwen3.5-9B: Error}
    \end{minipage}
    \vspace{-0.3em}
    \caption{Average attention error across different context lengths and sparsity levels for Qwen3.5-9B and Qwen3.5-4B on LOFT~\citep{lee2024longcontext}.}
    \label{fig:qwen3.5_error}
\end{figure*}

\section{Related Work}
\label{sec:related}

In the context of SSA, the extensive prior work can be organized into three broad and comprehensive research directions:
\\
\\
\noindent
\textbf{Head-level compression and KV eviction:} One line of work reduces KV-cache overhead by sharing or compressing KV heads. Multi-Query Attention (MQA)~\citep{shazeer2019fast} shares a single KV head across all query heads, while Grouped-Query Attention (GQA)~\citep{ainslie2023gqa} and Multi-Head Latent Attention (MLA)~\citep{deepseekai2024} provide more flexible compression to balance efficiency and model quality. A complementary line of work dynamically evicts less important KV entries during inference. Methods such as SnapKV~\citep{li2024snapkv}, DynamicKV~\citep{zhou2025dynamickv}, KVZip~\citep{kim2025kvzip}, KeyDiff~\citep{park2025keydiff}, CurDKV~\citep{sengupta2025valueguided}, and CapKV~\citep{yang2026rethinking} use different importance signals and eviction policies to reduce cache size and subsequent attention computation. These approaches are complementary to sparse attention: KV compression reduces the tokens or representations retained across decoding steps, whereas sparse attention dynamically selects which retained tokens to access for each query. The two can therefore be composed to further reduce memory and decoding costs.
\\
\\
\noindent
\textbf{Approximate top-$k$/top-$p$ methods:}
A substantial body of sparse-attention work focuses on finding the tokens with the largest attention scores. Methods such as Quest~\cite{pmlr-v235-tang24l},
PQCache~\cite{pqcache2025}, SqueezedAttention~\citep{hooper2024fast}, Adamas~\citep{zhang2024adamas}, SOCKET~\cite{joshi2026socket},
Sketch-and-Walk~\cite{le2026scout},
RetrievalAttention~\cite{liu2024retrievalattention},
Loki~\cite{loki2024lowrank},
HashAttention~\cite{desai2025hashattention}, and
Double Sparsity~\cite{yang2024double} employ different proxy scores or retrieval mechanisms to approximate the top-$k$ tokens under dense attention. These methods are effective at recovering high-attention ``heavy hitters,'' but a fixed top-$k$ selection discards all tokens outside the retrieval set. Consequently, when non-negligible attention mass is distributed across the residual tail, or when approximation error causes relevant tokens to fall just below the top-$k$ boundary, this discarded mass can create a substantial gap from dense attention at aggressive sparsity. Tactic~\cite{zhu2025topp} extends this retrieval paradigm from top-$k$ to top-$p$ selection by modeling the decay of attention scores with a power-law distribution and estimating the number of highest-scoring tokens required to satisfy a target attention-mass threshold.
\\
\\
\noindent
\textbf{Sampling-based and Hybrid methods:}
MagicPig~\cite{chen2024magicpig} addresses the limitations of deterministic top-$k$ attention by using LSH~\cite{IndykMotwani1998LSH} to approximately sample tokens and importance correction to estimate dense attention. vAttention~\cite{desai2025vattention} combines deterministic and stochastic selection through a verified hybrid estimator that retains top-$k$ tokens while uniformly sampling from the residual KV cache. KNN-Attention~\citep{haris2024knn} also explores top-$k$ plus sampling via lazy Gumbel sampling, but primarily targets training rather than inference acceleration for existing models. Unlike uniform hybrid sampling, which treats all residual tokens equally, SSA uses approximate attention scores to guide stochastic selection toward tokens more likely to carry residual attention mass (see fig.~\ref{fig:ssa_error_analysis}).

\section{Background and Motivation}
\label{sec:background}

\subsection{Sparse Attention}
We begin by describing the dense attention computation. For clarity, we consider a single query
$\mathbf{q}\in\mathbb{R}^{d}$ attending to a KV cache
$\mathbf{K}:=(\kb_1,\dots,\kb_N)^\top\in\RR{N}{d}$ and $\mathbf{V}:=(\mathbf{v}_1,\dots, \mathbf{v}_N)^\top\in\mathbb{R}^{N\times d}$.
\\
\\
\noindent
\textbf{Full Scaled Dot-Product Attention (SDPA):} Let
$
    s_i = \nicefrac{\mathbf{q}^\top \mathbf{k}_i}{\sqrt{d}}
$
be the similarity score between $\mathbf{q}$ and $\kb_i$, and $\mathbf{a}=(a_1,\dots, a_N)$ with $a_i =
    \nicefrac{w_i}{\sum_{j=1}^{N} w_j}$ be the corresponding attention weights where $  w_i = \exp(s_i)$.
We write SDPA output $\mathbf{o}\in\mathbb{R}^d$ as
\begin{equation}
    \mathbf{o}
    =
    \frac{
        \sum_{i=1}^{N} w_i \mathbf{v}_i
    }{
        \sum_{j=1}^{N} w_j
    }
    =
    \sum_{i=1}^{N} a_i \mathbf{v}_i=\Vb^\ts\mathbf{a}.
    \label{eq:dense-attention}
\end{equation}
\\
\\
\noindent
\textbf{Sparse Attention with deterministic top-$k$ Selection:} Let heavy set $\mathcal{H}\subseteq[N]$ denote a deterministic top-$k$ subset from the $N$ cached tokens. Sparse attention over the set $\mathcal{H}$ is defined as
    $\operatorname{o}_{\mathcal{H}}
    =
    \frac{
        \sum_{i\in\mathcal{H}} w_i \mathbf{v}_i
    }{
        \sum_{j\in\mathcal{H}} w_j
    }$.
While deterministic top-$k$ selection concentrates computation on tokens deemed
most relevant, tokens outside $\mathcal{H}$ do not contribute to the sparse output.
\\
\\
\noindent
\textbf{Sparse Attention with Importance Sampling:}
Let sampled set $\mathcal S\subseteq[N]$ be a non-empty random set of selected tokens.
For each token $i$, let $Z_i=\mathbf 1\{i\in\mathcal S\}$ indicate
whether it is selected, and let $\pi_i=\Pr(i\in\mathcal S)>0$ be its
inclusion probability. Importance-based sparse attention estimates the dense output by weighting each selected token by the inverse of its inclusion
probability:
\begin{equation}
    \widehat{\mathbf o}
    =
    \frac{\sum_{i=1}^{N}(Z_i/\pi_i)w_i\mathbf v_i}
         {\sum_{i=1}^{N}(Z_i/\pi_i)w_i}
    =
    \frac{\sum_{i\in\mathcal S}(w_i/\pi_i)\mathbf v_i}
         {\sum_{i\in\mathcal S}w_i/\pi_i}.
    \label{eq:importance-attention}
\end{equation}
The factors $1/\pi_i$ account for differences in how likely tokens are to be selected. The estimator in eq.~\ref{eq:importance-attention} is an instance of inverse-probability weighting. We next recall the standard draw-based formulation of importance sampling.

\subsection{Importance Sampling}
 Importance sampling~\citep{tokdar2010importance} is a Monte Carlo technique for estimating a quantity using samples drawn from a proposal distribution. Given a target distribution $p(x)$ and a function $f(x)$, consider the expectation $\mu = \mathbb{E}_{x\sim p}[f(x)] = \sum_x p(x)f(x).$ Instead of sampling directly from $p(x)$, importance sampling draws samples $x_1,\ldots,x_m \sim q(x)$ from a proposal distribution $q(x)$ and rewrites the expectation as
$
\mu
= \mathbb{E}_{x\sim q}
\left[
    \frac{p(x)}{q(x)}f(x)
\right].
$ This yields the Monte Carlo estimator

$$
\hat{\mu}
= \frac{1}{m}\sum_{j=1}^{m}
\frac{p(x_j)}{q(x_j)}f(x_j).
$$
\noindent
The ratio $p(x)/q(x)$ is the \emph{importance weight}, which corrects for sampling from $q$ rather than $p$. When $q$ places greater probability on samples with large contributions to the expectation, importance sampling can achieve lower variance than uniform sampling with the same sample budget.

\section{Introducing Stratified Sparse Attention (SSA)}

As already mentioned, deterministic top-$k$ attention~\citep{ge2024model,zhang2023h2o} can omit meaningful attention mass outside the selected set. 
Uniform sampling on residual tokens provides broader coverage, but does not use information about which remaining tokens are more likely to matter. We introduce \textbf{Stratified Sparse Attention (SSA)}, a unified framework that retains tokens with high proxy scores and samples the residual tokens according to those scores. SSA divides a fixed attention budget between two well-known complementary objectives: \emph{exploitation}, by deterministically retaining high-confidence heavy hitters, and \emph{exploration}, by sampling the residual context according to proxy attention scores. This score-aware stratification preserves heavy hitters while providing meaningful stochastic coverage toward more promising residual tokens. We next formalize the SSA mechanism described Algorithm~\ref{alg:ssa}.

\subsection{Final Algorithm}




SSA operates in three stages that divide the sparse-attention budget $k$ between deterministic exploitation and stochastic exploration.

\noindent
(i) \textbf{Proxy-based budget allocation (Steps 1-4):}
SSA converts the proxy scores $\tilde{s}_i$ into probabilities $q_i$ and finds
$\lambda$ such that
$\sum_{i=1}^{N}\min\{1,\lambda q_i\}=k$.
Tokens satisfying $\lambda q_i\geq 1$ form the deterministically
retained heavy set $\mathcal H$. The remaining $m=k-|\mathcal H|$ is allocated to importance sampling from the residual set
$\mathcal R=[N]\setminus\mathcal H$. Thus, the proxy distribution determines the split between deterministic allocation and residual sampling automatically; SSA does not require a
separate choice of heavy-set size.

\noindent
\noindent
(ii) \textbf{Score-guided residual sampling (Steps 5-10):}
SSA normalizes the proxy probabilities over $\mathcal R$ to obtain proposal probabilities $p_i=q_i/\sum_{j\in\mathcal R}q_j$, makes $m$ multinomial draws with
replacement, and retains the unique keys in $\mathcal S$. Let
$\mathcal K=\mathcal H\cup\mathcal S$, and define
$Z_i=\mathbf 1\{i\in\mathcal K\}$ to indicate whether token $i$ is selected. Its inclusion probability is $\pi_i=\Pr(Z_i=1)$.
For $i\in\mathcal H$, $\pi_i=1$. For $i\in\mathcal R$, the token is selected if it appears in at least one of the $m$ draws, giving
$\pi_i=1-(1-p_i)^m.$

\noindent
(iii) \textbf{Exact scoring and correction (Steps 11-15):}
For each token in $\mathcal K$, SSA computes
its exact score $\ell_i=\mathbf q^\top\mathbf k_i/\sqrt d$ and corrected
weight $\hat w_i=\exp(\ell_i)/\pi_i$. It returns the normalized weighted sum
$
    \hat{\mathbf o}.
$
Since repeated draws are retained only once, $|\mathcal K|\leq k$.

\begin{figure*}[t]
\centering

\begin{minipage}[t]{0.49\textwidth}
\vspace{0pt}

\begin{algorithm}[H]
\caption{Stratified Sparse Attention (SSA)}
\label{alg:ssa}
\begin{algorithmic}[1]
\\
\textbf{Input:} Proxy scores $\{\tilde{s}_i\}_{i=1}^{N}$,
query $\mathbf{q}$, keys $\{\mathbf{k}_i\}_{i=1}^{N}$,
values $\{\mathbf{v}_i\}_{i=1}^{N}$, total budget $k$

\vspace{1mm}
\State $q_i \gets
\frac{\exp(\tilde{s}_i)}
{\sum_{j=1}^N\exp(\tilde{s}_j)},
$

\vspace{1mm}
\State Find $\lambda>0$ s.t. $\sum_{i=1}^N \text{min}(1, \lambda q_i) = k$

\vspace{1mm}
\State Define heavy set $\mathcal{H} = \{i: \lambda q_i \geq 1 \}$,
residual set $\mathcal{R} = [N]\setminus\mathcal{H}$, heavy size $h =|\mathcal{H}|$,
sampling budget $m=k-h$

\vspace{1mm}
\State Normalize the proposal over $\mathcal{R}:
p_i = \frac{q_i}{\sum_{j\in\mathcal{R}} q_j} $ for $i \in \mathcal{R}$

\vspace{1mm}
\State Draw with replacement:
$J_1, J_2, \dots J_m \sim \mathbf{p}$

\vspace{1mm}
\State $\mathcal{S} \gets
\operatorname{unique}\{J_1, J_2, \dots, J_m\}$

\vspace{1mm}
\State $\pi_i\gets1-(1-p_i)^m$ for $i\in\mathcal{S}$


\vspace{0.8mm}
\State $\mathcal{K}\gets\mathcal{H}\cup\mathcal{S}$

\vspace{1mm}
\State $\pi_i \gets 1$ for $i \in \mathcal{H}$ 

\vspace{1mm}
\For{$i\in\mathcal{K}$}
    \State $\ell_i\gets
    \mathbf{q}^{\top}\mathbf{k}_i/\sqrt{d}$
    \State $\hat{w}_i\gets\exp(\ell_i)/\pi_i$
\EndFor


\State \Return $\displaystyle\hat{\mathbf{o}}
\gets
\frac{\sum_{i\in\mathcal{K}}\hat{w}_i\mathbf{v}_i}
     {\sum_{i\in\mathcal{K}}\hat{w}_i}$
\end{algorithmic}
\end{algorithm}

\end{minipage}
\hfill
\begin{minipage}[t]{0.46\textwidth}
\vspace{0pt}
\centering

\captionof{table}{
Average RULER-HARD performance at 50$\times$ attention sparsity
across models and context lengths. Full per-task results are
provided in the Appendix.}
\label{tab:ruler_hard_main}


\resizebox{\linewidth}{!}{
\begin{tabular}{llccc}
\toprule
\textbf{Model} & \textbf{Method}
& \textbf{32K} & \textbf{64K} & \textbf{128K} \\
\midrule

\multirow{5}{*}{Llama-3.1-8B}
& Dense      & 88.76 & 86.52 & 71.26 \\
& PQ         & 72.51 & 67.53 & 39.73 \\
& vAttn (PQ) & 66.71 & 66.22 & 39.45 \\
& Oracle-top-$k$  & 78.80 & 71.91 & \textbf{56.10} \\
& SSA        & \textbf{83.67} & \textbf{77.09} & 47.53 \\
\midrule

\multirow{5}{*}{Ministral-3-3B}
& Dense      & 86.69 & 82.19 & 81.49 \\
& PQ         & 70.81 & 69.88 & 66.48 \\
& vAttn (PQ) & 69.26 & 65.77 & 63.06 \\
& Oracle-top-$k$  & 72.81 & 75.09 & \textbf{71.87} \\
& SSA        & \textbf{82.05} & \textbf{75.70} & 69.65 \\
\midrule

\multirow{5}{*}{Ministral-3-8B}
& Dense      & 89.33 & 86.13 & 81.45 \\
& PQ         & 81.83 & 78.02 & 62.99 \\
& vAttn (PQ) & 78.69 & 75.19 & 63.78 \\
& Oracle-top-$k$  & \textbf{86.61} & \textbf{83.41} & \textbf{75.83} \\
& SSA        & 84.17 & 79.91 & 69.25 \\
\midrule

\multirow{5}{*}{Qwen3.5-4B}
& Dense      & 91.95 & 91.11 & 89.45 \\
& PQ         & 91.22 & 90.72 & 88.72 \\
& vAttn (PQ) & 90.56 & 90.61 & 89.28 \\
& Oracle-top-$k$  & 91.39 & 90.39 & 89.39 \\
& SSA        & \textbf{91.45} & \textbf{90.83} & \textbf{89.50} \\
\midrule

\multirow{5}{*}{Qwen3.5-9B}
& Dense      & 93.45 & 92.22 & 91.28 \\
& PQ         & 92.00 & 91.08 & 90.95 \\
& vAttn (PQ) & 91.67 & 91.30 & 90.61 \\
& Oracle-top-$k$  & 92.56 & 92.28 & 91.39 \\
& SSA        & \textbf{93.33} & \textbf{92.36} & \textbf{91.45} \\
\bottomrule
\end{tabular}}

\end{minipage}

\end{figure*}

\subsection{Theoretical Insights}\label{sec:theory}
We now characterize how accurately SSA approximates dense attention. Our analysis of Algorithm~\ref{alg:ssa} is designed for the budget parameter $k\in\{1,\dots,N-1\}$. When $k=N$, every token
is retained (see step~3), so SSA trivially returns dense attention exactly. The residual
set is then empty and requires no sampling analysis. For $k<N$,
the residual set is nonempty \ie, $m\ge1$. Let $a_i$ be the dense
attention weights from Section~\ref{sec:background}, and $\|\Vb\|_2$ denote the spectral
norm of the value matrix. In addition, we introduce the following quantities for our analysis,
\[
A_{\mathcal R}=\sum_{i\in\mathcal R}a_i,
\qquad
\bar a_i=\frac{a_i}{A_{\mathcal R}},
\qquad
\rho=\max_{i\in\mathcal R}\frac{\bar a_i}{p_i},
\]
where $A_{\mathcal R}$ is the true attention mass left for sampling, while $\rho$
measures how well the residual proposal covers that mass. For finite proxy scores,
$q_i>0$ and hence $p_i=q_i/\sum_{j\in\mathcal R}q_j>0$ for every
$i\in\mathcal R$. Since
$\sum_{i\in\mathcal R}p_i(\bar a_i/p_i)
=\sum_{i\in\mathcal R}\bar a_i=1$, the largest ratio must satisfy
$\rho\geq1$. The equality holds precisely when every ratio is equal to one,
\ie, $p_i=\bar a_i$ for all $i\in\mathcal R$.
\\
\\
\noindent
The budget-allocation rule also limits the effect of duplicate residual draws. Indeed, steps~3--4 of Algorithm~\ref{alg:ssa} give
$k=h+\lambda\sum_{i\in\mathcal R}q_i$ as
each heavy token contributes $1$ to the budget equation, and each
residual token contributes $\lambda q_i<1$. Since
$m=k-h$, it follows that
$m=\lambda\sum_{i\in\mathcal R}q_i$. Combining this with
$p_i=q_i/\sum_{j\in\mathcal R}q_j$ gives
$mp_i=\lambda q_i<1$ for every $i\in\mathcal R$. 
Here $mp_i$ is the expected number of times token $i$ appears among the $m$
draws. The largest expected multiplicity among the residual tokens is therefore defined as
\[
c:=\max_{i\in\mathcal R}mp_i<1.
\]
A small value of $c$ corresponds to a diffuse proposal with few repetitions,
whereas a value closer to one indicates a more concentrated proposal and a
greater possibility of repeated draws.
\\
\\
\noindent
Now, to use the proxy-quality parameter $\rho$ in the error bound, we need to
relate the proposal probability $p_i$ to the inclusion probability
$\pi_i$ used by Algorithm~\ref{alg:ssa}. A residual token is retained if it appears in at
least one of the $m$ draws, so $\pi_i=1-(1-p_i)^m$. 
To account for duplicate draws, we define
\[
\kappa(c):=\frac{1-e^{-c}}{c}.
\]
The following lemma quantifies their effect on the inclusion probabilities.
%
%
\begin{lemma}
\label{lem:deduplication}
Let $m\ge1$, let $p_i>0$ for $i\in\mathcal R$, and let $c$ and $\kappa(c)$ be defined above.
For every $i\in\mathcal R$,
$\kappa(c)mp_i
\le \pi_i
\le mp_i$.
Consequently,
$\frac{a_i}{\pi_i}
\le
\frac{A_{\mathcal R}\rho}{\kappa(c)m}$.
\end{lemma}
\noindent
Now, we are ready to use the lemma
to bound SSA's approximation error.
\begin{theorem}
\label{thm:ssa-hoeffding}
For any $\delta\in(0,1)$, let
$\epsilon_D(\delta)
:=
\frac{2\sqrt{2}\,A_{\mathcal R}\rho}{\kappa(c)}
\sqrt{\frac{\log(4/\delta)}{m}}$.
If $\epsilon_D(\delta)<1$, then, with probability at least
$1-\delta$ over the residual sampling,
\begin{equation}
\left\|
\widehat{\mathbf o}-\mathbf o
\right\|_2
\leq
\frac{
4\sqrt{2}\,A_{\mathcal R}\rho
}{
\kappa(c)\bigl(1-\epsilon_D(\delta)\bigr)
}
\sqrt{\frac{\log(4/\delta)}{m}}\,
\|\mathbf V\|_2.
\label{eq:ssa-hoeffding}
\end{equation}
\end{theorem}
\noindent
Since $\kappa(c)>1-e^{-1}$, deduplication contributes only a universal
constant factor to the approximation and sample-complexity bounds. The
following corollary gives a direct $(\varepsilon,\delta)$ guarantee and a
sufficient sample size $m
=
\Ocal\left(
A_{\mathcal R}^{2}\rho^{2}\varepsilon^{-2}
\log\frac{1}{\delta}
\right)$ that automatically ensures
$\epsilon_D(\delta)<1$.

\begin{corollary}
\label{cor:ssa-sample-complexity}
Consider Algorithm~\ref{alg:ssa} with an integer budget
$1\le k<N$ and finite attention and proxy scores. Let
$\varepsilon,\delta\in(0,1)$. If
$m
\geq
21A_{\mathcal R}^{2}\rho^{2}
\frac{(2+\varepsilon)^2}{\varepsilon^2}
\log\frac{4}{\delta}$,
then, with probability at least $1-\delta$,
\begin{equation}
\left\|
\widehat{\mathbf o}-\mathbf o
\right\|_2
\leq
\varepsilon\|\mathbf V\|_2.
\label{eq:ssa-epsilon-delta}
\end{equation}
\end{corollary}

\begin{remark}

Note that the approximation error in Theorem~\ref{thm:ssa-hoeffding} scales as
$\nicefrac{A_{\mathcal R}\rho}{\sqrt m}$.
The residual mass $A_{\mathcal R}$ decreases when the retained heavy tokens
capture more of the true attention mass, while $\rho$ increases when the proxy
underestimates an important residual token. Increasing the residual budget
$m$ reduces the stochastic error at the standard Monte Carlo rate
$m^{-1/2}$. For a fixed total budget, SSA therefore balances the reduction in
$A_{\mathcal R}$ obtained from retaining more heavy tokens against the larger
sampling budget obtained from retaining fewer of them.
\end{remark}
\begin{remark}
In both Theorem~\ref{thm:ssa-hoeffding} and Corollary~\ref{cor:ssa-sample-complexity}, we express the bounds using $\|\Vb\|_2$. The proofs only require
a bound on the individual value-vector norms, so the same results
hold with $\max_{i\in[N]}\|\mathbf v_i\|_2$ in its place, without
changing the constants.
\end{remark}

\subsection{Kernel Design}
\label{sec:kernel}

At long contexts, decode attention is memory-bandwidth bound because each step reads the entire KV cache from HBM~\citep{shah2024flashattention, kwon2023efficient}. SSA reduces this memory traffic by selecting a sparse subset of tokens before loading their KV states. While Algorithm~\ref{alg:ssa} determines the division between the heavy and residual strata adaptively, our optimized GPU implementation uses a fixed heavy fraction. Developing an equally efficient implementation of the adaptive algorithm remains future work. As shown in Table~\ref{tab:ssa_theory_practical}, the practical fixed-budget variant closely matches the performance of the theoretical formulation. The residual proposal, inclusion-probability correction, and resulting attention estimator remain unchanged. Our SSA decode step consists of four stages, where only the first three scan the full context using a few bytes per token.
\\
\\
\noindent
\textbf{Proxy scoring:} Each key is represented by a 16-byte PQ code (16 subspaces, 256 centroids)~\citep{pqcache2025}, compared to 1\,KiB for its fp16 key-value pair at head dimension 256. For each query, a lookup table computes the proxy score $\tilde{s}_i$ for every token. Scores are scaled and masked for sink and local tokens within the same kernel, avoiding an additional pass.
\\
\\
\noindent
\textbf{Heavy-set selection:} We select $\mathcal{H}$ using radix top-$k$~\citep{radixtopk}, which returns an unsorted index set since the estimator in eq.~\ref{eq:importance-attention} is invariant to the ordering of selected tokens.
\\
\\
\noindent
\textbf{Residual sampling:} Direct sampling from the full residual distribution becomes expensive at long context lengths, as constructing and sampling from the distribution requires processing probabilities across the entire residual KV cache. To reduce this overhead, we use a hierarchical sampler that progressively narrows the search space. Each draw first selects a 1024-token chunk, then a 32-token sub-tile, and finally scans at most 32 tokens to select a key. We remove duplicate samples with a per-row hash table and reweight each unique sample by $1/\pi_i$. In practice, we clamp $\pi_i$ at $10^{-4}$ to avoid excessively large weights, at the cost of introducing a small bias. Finally, we use a counter-based random number generator to obtain fresh samples across replays while retaining compatibility with CUDA graph capture.
\\
\\
\noindent
\textbf{Sparse decode:} The final stage gathers the keys and values in $\mathcal{H}\cup\mathcal{S}$ from the paged KV cache and evaluates eq.~\ref{eq:importance-attention} using an online softmax. To exploit more parallelism, we divide each query head across multiple thread blocks and combine their partial softmax results using a numerically stable log-sum-exp reduction. At 1M context and batch size~1, using 64 splits is 1.58-1.71$\times$ faster than using 32 splits across sparsity levels. Increasing the number of splits to 128 or 256, however, reduces performance because the additional parallelism no longer offsets the increased launch and merge overhead.

\section{Experiments}
\label{sec:exp}

We conduct an extensive evaluation of SSA against Oracle-top-$k$, deterministic top-$k$, pure sampling, and hybrid baselines across multiple datasets, models, context lengths, and generation lengths. In addition, note that we slightly overload the notation $h$. In Algorithm~\ref{alg:ssa} and the theoretical analysis, $h=|\mathcal H|$ denotes the number of heavy tokens, whereas in the implementation and experimental discussion, $h\in[0,1]$ denotes the fraction of the total budget assigned to the heavy set. We describe the evaluation setup below.

\subsection{Setup}
\label{sec:exp-setup}

\textbf{Baselines:} We compare SSA against three representative sparse-attention baselines: vAttention~\citep{desai2025vattention}, PQCache~\citep{pqcache2025}, and Oracle-top-$k$. Oracle-top-$k$ performs exact attention-score computation before selecting the top-$k$ tokens and is therefore impractical for acceleration; we include it as an upper-bound reference for deterministic top-$k$ selection. For all methods, we use 128 sink tokens and 128 local tokens, along with the hyperparameters recommended by the respective authors. Additional comparisons with MagicPig~\citep{chen2024magicpig}, HashAttention~\citep{desai2025hashattention}, Double Sparsity~\citep{yang2024double}, and Quest~\citep{pmlr-v235-tang24l} are provided in Table~\ref{tab:ruler_hard_32k_all_models} in the Appendix. These results motivate our use of PQCache, the strongest approximate top-$k$ baseline in this comparison, to construct SSA's proposal distribution throughout our experiments.
\\
\\
\noindent
\textbf{Datasets and Models:} We evaluate SSA on four benchmark suites: RULER (32K-128K)~\citep{hsieh2024ruler}, LongBench (10K-16K)~\citep{bai2024longbench}, LOFT (32K-128K)~\citep{lee2024longcontext}, and AIME (32K token generation)~\citep{dekoninck2026matharena}. Our evaluation spans five models: Llama-3.1-8B-Instruct~\citep{llama31}, Qwen3.5-4B and Qwen3.5-9B~\citep{qwen35blog}, and Ministral-3-3B-Instruct-2517 and Ministral-3-8B-Instruct-2517~\citep{liu2026ministral3}. Dataset subsets and additional evaluation details are provided in Table~\ref{tab:datasets_info} in Appendix.

\begin{table}[t]
\centering
\caption{\textbf{Kernel optimizations make SSA faster than dense decoding even at batch size 1.} We report decode-time attention-layer latency at 1M context length on a single H200 GPU. \emph{Reference} implements the same algorithm without the kernel optimizations described in Section~\ref{sec:kernel}. At batch size 1, it reaches only 0.37-0.47$\times$ the speed of FlashInfer and surpasses FlashInfer only at batch size 8. In contrast, optimized SSA is faster than FlashInfer across all measured batch sizes and sparsity levels. FlashInfer latencies across the two runs agree within 1.2\%.}
\label{tab:decode_1m}
\setlength{\tabcolsep}{5pt}
\begin{tabular}{llrrrcc}
\toprule
& & \multicolumn{3}{c}{\textbf{Latency (ms)}} & \multicolumn{2}{c}{\textbf{SSA speedup}} \\
\cmidrule(lr){3-5}\cmidrule(lr){6-7}
\textbf{Sparsity} & \textbf{Batch} & \textbf{FlashInfer} & \textbf{Reference} & \textbf{SSA}
& \textbf{vs.\ FlashInfer} & \textbf{vs.\ reference} \\
\midrule
\multirow{3}{*}{50$\times$} & 1  & 1.41  & 2.99  & 0.79  & 1.79$\times$ & 3.80$\times$ \\
                            & 8  & 11.90 & 7.18  & 4.27  & 2.78$\times$ & 1.68$\times$ \\
                            & 16 & 24.47 & 12.57 & 8.26  & 2.96$\times$ & 1.52$\times$ \\
\midrule
\multirow{3}{*}{20$\times$} & 1  & 1.41  & 3.34  & 1.01  & 1.40$\times$ & 3.32$\times$ \\
                            & 8  & 11.91 & 8.74  & 5.76  & 2.07$\times$ & 1.52$\times$ \\
                            & 16 & 24.53 & 15.51 & 11.32 & 2.17$\times$ & 1.37$\times$ \\
\midrule
\multirow{3}{*}{10$\times$} & 1  & 1.42  & 3.83  & 1.30  & 1.10$\times$ & 2.95$\times$ \\
                            & 8  & 12.01 & 11.06 & 7.91  & 1.52$\times$ & 1.40$\times$ \\
                            & 16 & 24.55 & 19.82 & 15.81 & 1.55$\times$ & 1.25$\times$ \\
\bottomrule
\end{tabular}
\end{table}

\subsection{Effectiveness of SSA}
\label{sec:exp-accuracy}

\textbf{RULER-HARD:} Table~\ref{tab:ruler_hard_main} reports average RULER-HARD scores at 50$\times$ attention sparsity. SSA consistently outperforms both PQ and vAttention(PQ) across all models and context lengths. Since these methods use the same PQ scores and attention budget, the comparison highlights the benefit of allocating part of the budget to score-guided sampling rather than exclusively to top-$k$ retrieval or uniform sampling. The gains are particularly pronounced when deterministic top-$k$ struggles. At 32K, SSA improves Llama-3.1-8B from 72.51 (PQ) to 83.67, approaching the dense score of 88.76. Similarly, on Ministral-3-3B, SSA improves the score from 70.81 to 82.05, compared with 86.69 for dense attention. Oracle-top-$k$, which selects tokens using exact attention scores, remains stronger at 128K for Llama-3.1-8B and Ministral-3-3B, and across all context lengths for Ministral-3-8B. On Qwen3.5-4B and Qwen3.5-9B, all sparse methods remain within two points of dense attention.
\\
\\
\noindent
\textbf{AIME 2025:} Table~\ref{tab:aime_and_lb} reports avg@16 accuracy for long-form generation. SSA remains robust as sparsity increases, with particularly large gains in the aggressive-sparsity regime. On Qwen3.5-9B at 50$\times$ sparsity, SSA achieves 59.8, compared with 48.3 for oracle top-$k$, 36.7 for PQ, and 35.8 for vAttention(PQ). At 100$\times$ sparsity, SSA retains an accuracy of 51.5, while vAttention(PQ) drops to 3.5. At lower sparsity, vAttention(PQ) performs better at 5$\times$ and at 10$\times$ on Qwen3.5-4B. We further study the effect of the heavy fraction $h$ in Table~\ref{tab:aime_heavy_size}. Hybrid allocations with $h\in[0.2,0.8]$ achieve 61.0-64.0, substantially outperforming both pure sampling ($h=0$), which achieves 40.2-42.5, and pure top-$k$ ($h=1$), which achieves 49.4-53.8. These results demonstrate the benefit of combining deterministic retrieval with score-guided importance sampling, particularly under aggressive sparsity.
\\
\\
\noindent
\textbf{LOFT and LongBench:} On LOFT with Qwen3.5-9B and Qwen3.5-4B
(see fig.~\ref{fig:qwen3.5_error}), SSA has the lowest attention error among the non-oracle methods at every context length and sparsity level, while subspan-EM varies by at most 0.04 across methods (see fig.~\ref{fig:loft_results}). On LongBench with Qwen3.5-27B
(Table~\ref{tab:aime_and_lb}), all sparse methods stay within
one point of dense attention at both 50$\times$ and 20$\times$ sparsity.

\begin{table}[t]
\centering

\begin{minipage}[t]{0.7\linewidth}
    \centering
    
    \resizebox{\linewidth}{!}{%
    \begin{tabular}{c|cccc|cccc}
        \toprule
        & \multicolumn{4}{c|}{\textbf{Qwen3.5-9B}}
        & \multicolumn{4}{c}{\textbf{Qwen3.5-4B}} \\
        \textbf{Sparsity}
        & \textbf{Oracle-top-$k$} & \textbf{PQ} & \textbf{vAttn(PQ)} & \textbf{SSA}
        & \textbf{Oracle-top-$k$} & \textbf{PQ} & \textbf{vAttn(PQ)} & \textbf{SSA} \\
        \midrule
        Dense
        & \multicolumn{4}{c|}{66.3}
        & \multicolumn{4}{c}{54.2} \\
        \midrule
        5$\times$
        & 63.8 & 63.5 & \textbf{65.4} & 61.0
        & 49.8 & 49.4 & \textbf{54.4} & 52.7 \\

        10$\times$
        & 59.4 & 53.8 & 62.7 & \textbf{64.0}
        & 47.9 & 41.5 & \textbf{51.9} & 51.0 \\

        20$\times$
        & 55.4 & 49.4 & 58.3 & \textbf{63.5}
        & 42.7 & 33.8 & 50.2 & \textbf{53.1} \\

        50$\times$
        & 48.3 & 36.7 & 35.8 & \textbf{59.8}
        & 36.5 & 25.4 & 27.5 & \textbf{50.0} \\

        100$\times$
        & 42.5 & 28.1 & 3.5 & \textbf{51.5}
        & 31.9 & 13.1 & 1.7 & \textbf{42.5} \\
        \bottomrule
    \end{tabular}%
    }
\end{minipage}
\hfill
\begin{minipage}[t]{0.28\linewidth}
    \centering
    \resizebox{\linewidth}{!}{%
    \begin{tabular}{lcc}
        \toprule
        \textbf{Method} & \textbf{50$\times$} & \textbf{20$\times$} \\
        \midrule
        Dense          & 41.56 & 41.56 \\
        \midrule
        PQ             & 40.60 & 40.98 \\
        vAttn (PQ)     & 40.39 & 41.47 \\
        SSA            & 40.92 & \textbf{41.56} \\
        Oracle-top-$k$ & \textbf{41.25} & 41.40 \\
        \bottomrule
    \end{tabular}%
    }
\end{minipage}
\caption{AIME 2025 avg@16 accuracy (left). LongBench performance with Qwen3.5-27B (right).}
\label{tab:aime_and_lb}
\end{table}

\subsection{Decode Efficiency}
\label{sec:exp-efficiency}

Sparse attention accelerates decoding only when its selection overhead is smaller than the memory traffic it saves. We evaluate this trade-off for SSA across context lengths and batch sizes, and separately measure the sampling overhead absent in top-$k$ methods.


\begin{figure}[t]
    \centering
    \includegraphics[width=\linewidth]{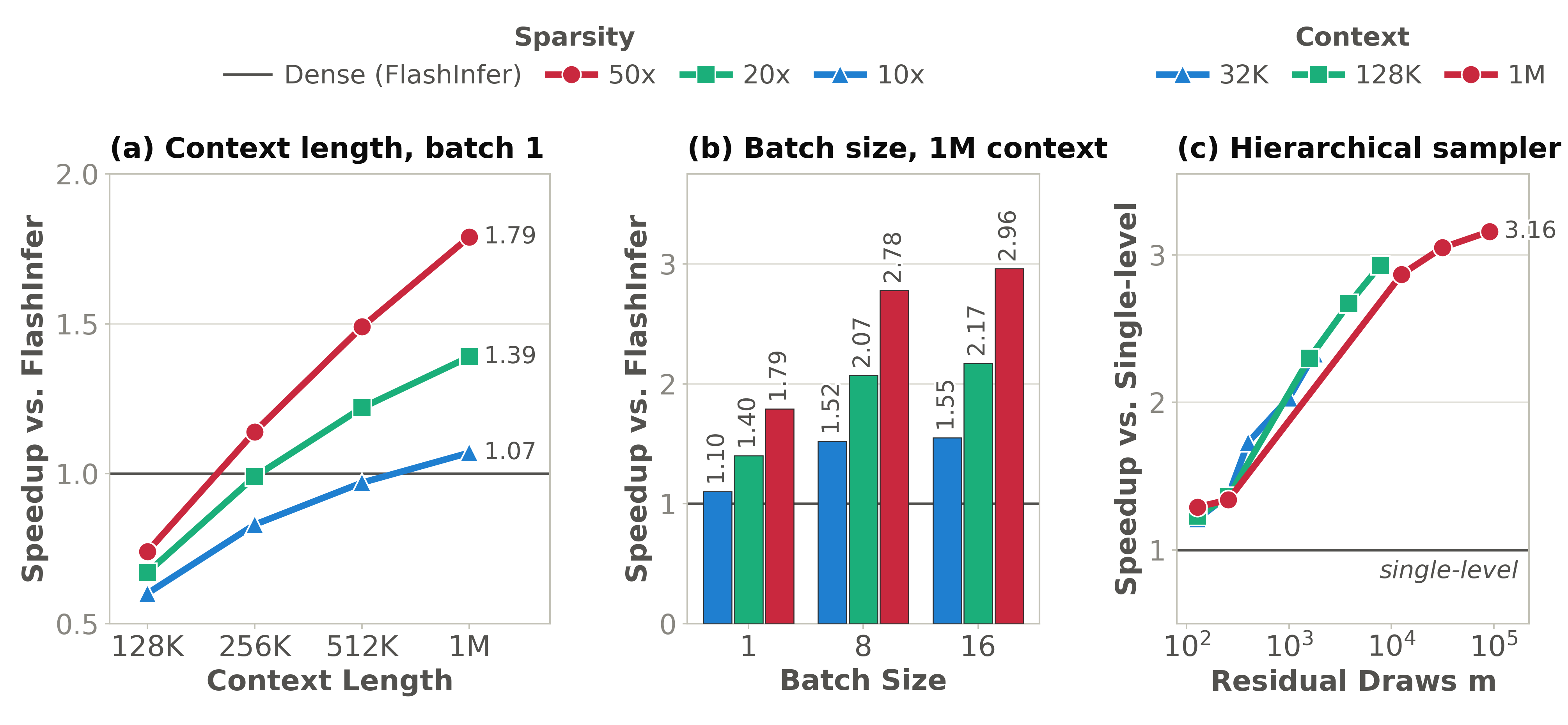}
    \caption{\textbf{SSA accelerates decode at long contexts, with larger gains at higher batch sizes.} Decode attention speedup on a single H200 (16 query heads, 4 KV heads, head dimension 256, fp16 KV). \textbf{(a)} At batch size 1, SSA overtakes FlashInfer beyond 256K context and reaches 1.79$\times$ speedup at 1M and 50$\times$ sparsity. \textbf{(b)} At 1M context, speedup increases with batch size, reaching 2.96$\times$ at batch 16 and 50$\times$ sparsity. \textbf{(c)} Overhead of two-level residual sampling relative to 1,024-token chunk-only sampling as the number of draws $m$ increases.}
    \label{fig:decode_efficiency}
\end{figure}

\noindent
\textbf{Efficiency analysis:} SSA becomes increasingly advantageous as context length and batch size grow. At batch size 1, SSA is slower than FlashInfer at 128K, achieving approximately 0.60-0.74$\times$ its speed. As context length increases, this overhead is increasingly amortized, and SSA reaches 1.79$\times$, 1.39$\times$, and 1.07$\times$ speedups at 1M for 50$\times$, 20$\times$, and 10$\times$ sparsity, respectively (see fig.~\ref{fig:decode_efficiency}a). The gains further increase with batch size, reaching 2.96$\times$ at batch 16 for 50$\times$ sparsity (see fig.~\ref{fig:decode_efficiency}b); our kernel optimizations are particularly important at batch 1, improving SSA by 2.95-3.80$\times$ over the reference implementation (Table~\ref{tab:decode_1m}). Finally, although residual sampling introduces overhead absent from top-$k$ methods, our two-level sampler reduces this cost by 1.21-3.16$\times$ over chunk-only routing as the number of draws $m$ increases (see fig.~\ref{fig:decode_efficiency}c), keeping sampling latency at 0.30-1.18\,ms at 1M context and $h=0.4$ across the evaluated sparsity levels.

\begin{figure}[H]
    \centering
    \includegraphics[width=0.9\linewidth]{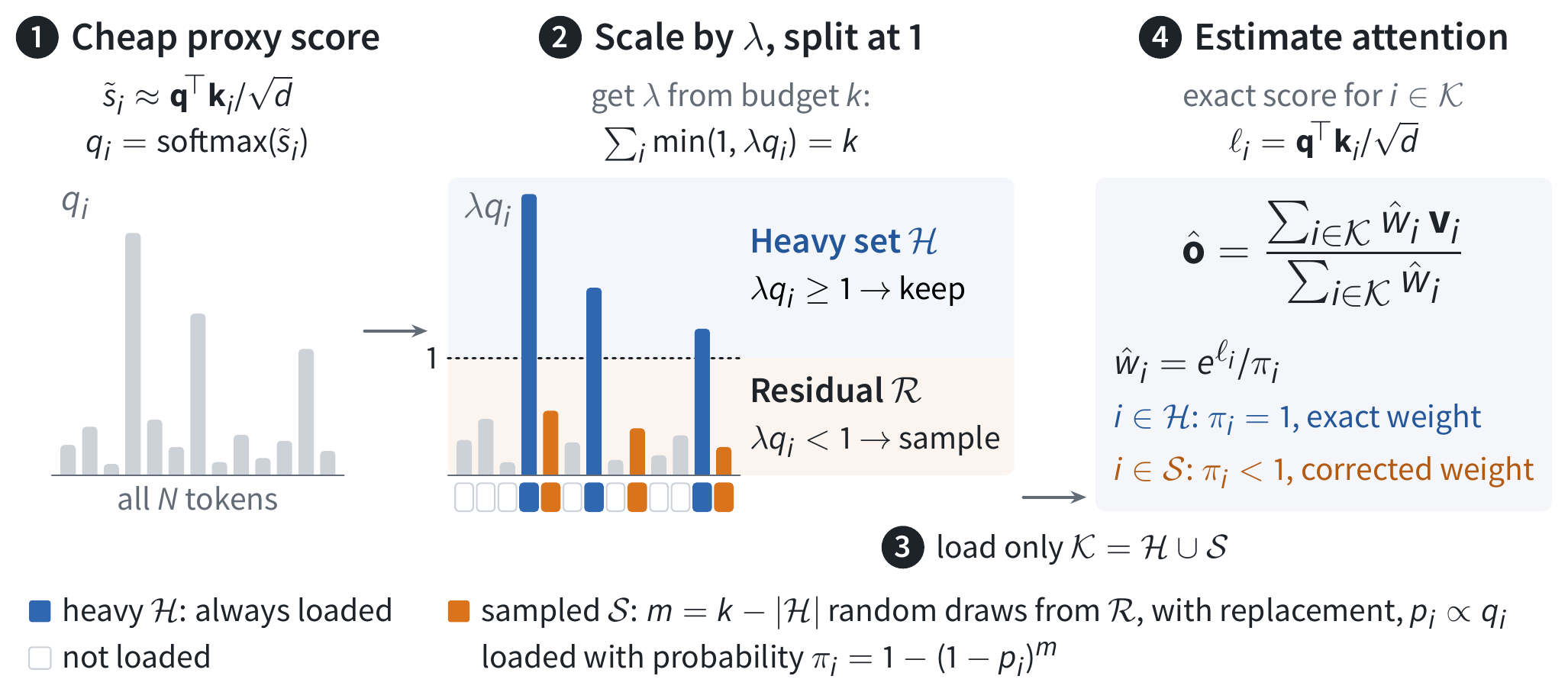}
    \caption{Illustration of Stratified Sparse Attention algorithm.}
    \label{fig:placeholder}
\end{figure}

\section{Conclusion}
We introduce Stratified Sparse Attention (SSA), a unified framework encompassing top-$k$, pure sampling, and hybrid sparse attention while substantially outperforming existing methods and reducing the accuracy gap to dense attention. SSA is compatible with any approximate top-$k$ method that provides proxy scores, allowing it to directly benefit from advances in approximate retrieval. An important direction for future work is to make SSA verifiable by incorporating user-controlled guarantees on approximation quality.

\section*{Acknowledgments}
The work was primarily supported by Rice Ken Kennedy Institute (K2I) Generative AI Cluster Funding. The Sky Computing Lab is supported by gifts from Accenture, Algorithmic SuperIntelligence Labs, Amazon, AMD, Anyscale, Broadcom, cmpnd, Google, IBM, Intel, Intesa Sanpaolo, Lambda, Lightspeed, Mirendil, NVIDIA, Samsung SDS, and VESSL.

\setlength{\bibsep}{6pt}
\bibliographystyle{abbrvnat}
{
\bibliography{bibliography.bib}

\begin{thebibliography}{50}
\providecommand{\natexlab}[1]{#1}
\providecommand{\url}[1]{\texttt{#1}}
\expandafter\ifx\csname urlstyle\endcsname\relax
  \providecommand{\doi}[1]{doi: #1}\else
  \providecommand{\doi}{doi: \begingroup \urlstyle{rm}\Url}\fi

\bibitem[Ainslie et~al.(2023)Ainslie, Lee-Thorp, de~Jong, Zemlyanskiy, Lebron, and Sanghai]{ainslie2023gqa}
J.~Ainslie, J.~Lee-Thorp, M.~de~Jong, Y.~Zemlyanskiy, F.~Lebron, and S.~Sanghai.
\newblock {GQA: Training Generalized Multi-Query Transformer Models from Multi-Head Checkpoints}.
\newblock In \emph{The 2023 Conference on Empirical Methods in Natural Language Processing}, 2023.

\bibitem[Bai et~al.(2024)Bai, Lv, Zhang, Lyu, Tang, Huang, Du, Liu, Zeng, Hou, Dong, Tang, and Li]{bai2024longbench}
Y.~Bai, X.~Lv, J.~Zhang, H.~Lyu, J.~Tang, Z.~Huang, Z.~Du, X.~Liu, A.~Zeng, L.~Hou, Y.~Dong, J.~Tang, and J.~Li.
\newblock {LongBench: A Bilingual, Multitask Benchmark for Long Context Understanding}.
\newblock In \emph{Proceedings of the 62nd Annual Meeting of the Association for Computational Linguistics}, Aug. 2024.

\bibitem[Cao et~al.(2026)Cao, Yin, Dhingra, and Zhou]{cao2026coding}
W.~Cao, X.~Yin, B.~Dhingra, and S.~Zhou.
\newblock {Coding Agents are Effective Long-Context Processors}.
\newblock \emph{arXiv preprint arXiv:2603.20432}, 2026.

\bibitem[Chen et~al.(2026)Chen, Wang, and Qu]{chen2026horizon}
M.~Chen, L.~Wang, and B.~Qu.
\newblock {The Horizon Gap: Planning, Memory, Execution, Training, and Evaluation for Long-Horizon LLM Agents}.
\newblock \emph{arXiv preprint arXiv:2608.06663}, 2026.

\bibitem[Chen et~al.(2024)Chen, Hu, Huang, Jiang, Ji, Jiang, Jiang, Liu, Liu, Li, Lian, Meng, Peng, Sun, Shi, Wang, Wang, Wang, Wang, Xuan, Xia, Yang, Yang, Zhang, Zhou, and Zhang]{chen2024deep}
X.~Chen, X.~Hu, Y.~Huang, H.~Jiang, W.~Ji, Y.~Jiang, Y.~Jiang, B.~Liu, H.~Liu, X.~Li, X.~Lian, G.~Meng, X.~Peng, H.~Sun, L.~Shi, B.~Wang, C.~Wang, J.~Wang, T.~Wang, J.~Xuan, X.~Xia, Y.~Yang, Y.~Yang, L.~Zhang, Y.~Zhou, and L.~Zhang.
\newblock Deep learning-based software engineering: Progress, challenges, and opportunities.
\newblock \emph{CoRR}, abs/2410.13110, 2024.

\bibitem[Chen et~al.(2025)Chen, Sadhukhan, Ye, Zhou, Zhang, Nolte, Tian, Douze, Bottou, Jia, and Chen]{chen2024magicpig}
Z.~Chen, R.~Sadhukhan, Z.~Ye, Y.~Zhou, J.~Zhang, N.~Nolte, Y.~Tian, M.~Douze, L.~Bottou, Z.~Jia, and B.~Chen.
\newblock {MagicPIG: LSH Sampling for Efficient LLM Generation}.
\newblock In \emph{International Conference on Learning Representations}, 2025.

\bibitem[DeepSeek-AI(2024)]{deepseekai2024}
DeepSeek-AI.
\newblock {DeepSeek-V2: A Strong, Economical, and Efficient Mixture-of-Experts Language Model}.
\newblock \emph{arXiv preprint arXiv:2405.04434}, 2024.

\bibitem[Desai et~al.(2025)Desai, Yang, Cuadron, Klimovic, Zaharia, Gonzalez, and Stoica]{desai2025hashattention}
A.~Desai, S.~Yang, A.~Cuadron, A.~Klimovic, M.~Zaharia, J.~E. Gonzalez, and I.~Stoica.
\newblock {HashAttention: Semantic Sparsity for Faster Inference}.
\newblock In \emph{International Conference on Machine Learning}, 2025.

\bibitem[Desai et~al.(2026)Desai, Agrawal, Yang, Cuadron, Schroeder, Zaharia, Gonzalez, and Stoica]{desai2025vattention}
A.~Desai, K.~K. Agrawal, S.~Yang, A.~Cuadron, L.~G. Schroeder, M.~Zaharia, J.~E. Gonzalez, and I.~Stoica.
\newblock {vAttention: Verified Sparse Attention via Sampling}.
\newblock In \emph{The Fourteenth International Conference on Learning Representations}, 2026.
\newblock URL \url{https://openreview.net/forum?id=zzTDulLys0}.

\bibitem[Dosovitskiy et~al.(2021)Dosovitskiy, Beyer, Kolesnikov, Weissenborn, Zhai, Unterthiner, Dehghani, Minderer, Heigold, Gelly, Uszkoreit, and Houlsby]{dosovitskiy2021image}
A.~Dosovitskiy, L.~Beyer, A.~Kolesnikov, D.~Weissenborn, X.~Zhai, T.~Unterthiner, M.~Dehghani, M.~Minderer, G.~Heigold, S.~Gelly, J.~Uszkoreit, and N.~Houlsby.
\newblock An image is worth 16x16 words: Transformers for image recognition at scale.
\newblock In \emph{International Conference on Learning Representations}, 2021.

\bibitem[Ge et~al.(2024)Ge, Zhang, Liu, Zhang, Han, and Gao]{ge2024model}
S.~Ge, Y.~Zhang, L.~Liu, M.~Zhang, J.~Han, and J.~Gao.
\newblock {Model Tells You What to Discard: Adaptive KV Cache Compression for LLMs}.
\newblock In \emph{International Conference on Learning Representations}, 2024.

\bibitem[Grattafiori et~al.(2024)Grattafiori, Dubey, Jauhri, Pandey, Kadian, Al-Dahle, Letman, Mathur, Schelten, Vaughan, et~al.]{llama31}
A.~Grattafiori, A.~Dubey, A.~Jauhri, A.~Pandey, A.~Kadian, A.~Al-Dahle, A.~Letman, A.~Mathur, A.~Schelten, A.~Vaughan, et~al.
\newblock {The Llama 3 Herd of Models}.
\newblock \emph{arXiv preprint arXiv:2407.21783}, 2024.

\bibitem[Haris(2025)]{haris2024knn}
T.~Haris.
\newblock {kNN Attention Demystified: A Theoretical Exploration for Scalable Transformers}.
\newblock In \emph{The Thirteenth International Conference on Learning Representations}, 2025.

\bibitem[Hooper et~al.(2025)Hooper, Kim, Mohammadzadeh, Maheswaran, Zhao, Paik, Mahoney, Keutzer, and Gholami]{hooper2024fast}
C.~R.~C. Hooper, S.~Kim, H.~Mohammadzadeh, M.~Maheswaran, S.~Zhao, J.~Paik, M.~W. Mahoney, K.~Keutzer, and A.~Gholami.
\newblock {Squeezed Attention: Accelerating Long Context Length LLM Inference}.
\newblock In \emph{Proceedings of the 63rd Annual Meeting of the Association for Computational Linguistics}, 2025.

\bibitem[Hou et~al.(2024)Hou, Zhao, Liu, Yang, Wang, Li, Luo, Lo, Grundy, and Wang]{hou2024large}
X.~Hou, Y.~Zhao, Y.~Liu, Z.~Yang, K.~Wang, L.~Li, X.~Luo, D.~Lo, J.~Grundy, and H.~Wang.
\newblock {Large Language Models for Software Engineering: A Systematic Literature Review}.
\newblock \emph{ACM Trans. Softw. Eng. Methodol.}, 33\penalty0 (8), 2024.

\bibitem[Hsieh et~al.(2024)Hsieh, Sun, Kriman, Acharya, Rekesh, Jia, Zhang, and Ginsburg]{hsieh2024ruler}
C.-P. Hsieh, S.~Sun, S.~Kriman, S.~Acharya, D.~Rekesh, F.~Jia, Y.~Zhang, and B.~Ginsburg.
\newblock {RULER: What{\textquoteright}s the Real Context Size of {Y}our Long-Context Language Models?}
\newblock In \emph{First Conference on Language Modeling (COLM)}, 2024.

\bibitem[Indyk and Motwani(1998)]{IndykMotwani1998LSH}
P.~Indyk and R.~Motwani.
\newblock {Approximate Nearest Neighbors: Towards Removing the Curse of Dimensionality}.
\newblock In \emph{Proceedings of the thirtieth annual ACM symposium on Theory of computing}, pages 604--613, 1998.

\bibitem[Joshi et~al.(2026)Joshi, Chowdhury, Bellinger, Kanakamedala, Singh, Le, Desai, and Shrivastava]{joshi2026socket}
S.~Joshi, A.~Chowdhury, W.~Bellinger, A.~Kanakamedala, E.~Singh, H.~A.~D. Le, A.~Desai, and A.~Shrivastava.
\newblock {SOCKET: SOft Collision Kernel EsTimator for Sparse Attention}.
\newblock \emph{arXiv preprint arXiv:2602.06283}, 2026.

\bibitem[Kim et~al.(2025)Kim, Kim, Kwon, Lee, Yun, and Song]{kim2025kvzip}
J.-H. Kim, J.~Kim, S.~Kwon, J.~W. Lee, S.~Yun, and H.~O. Song.
\newblock {KVzip: Query-Agnostic {KV} Cache Compression with Context Reconstruction}.
\newblock In \emph{The Thirty-ninth Annual Conference on Neural Information Processing Systems}, 2025.

\bibitem[Kwon et~al.(2023)Kwon, Li, Zhuang, Sheng, Zheng, Yu, Gonzalez, Zhang, and Stoica]{kwon2023efficient}
W.~Kwon, Z.~Li, S.~Zhuang, Y.~Sheng, L.~Zheng, C.~H. Yu, J.~E. Gonzalez, H.~Zhang, and I.~Stoica.
\newblock {Efficient Memory Management for Large Language Model Serving with PagedAttention}.
\newblock In \emph{Proceedings of the 29th Symposium on Operating Systems Principles}, 2023.

\bibitem[Le et~al.(2026)Le, Joshi, Yang, Xu, and Shrivastava]{le2026scout}
H.~A.~D. Le, S.~Joshi, Z.~Yang, Z.~Xu, and A.~Shrivastava.
\newblock Scout before you attend: Sketch-and-walk sparse attention for efficient {LLM} inference.
\newblock In \emph{Forty-third International Conference on Machine Learning}, 2026.

\bibitem[Lee et~al.(2024)Lee, Chen, Dai, Dua, Sachan, Boratko, Luan, Arnold, Perot, Dalmia, Hu, Lin, Pasupat, Amini, Cole, Riedel, Naim, Chang, and Guu]{lee2024longcontext}
J.~Lee, A.~Chen, Z.~Dai, D.~Dua, D.~S. Sachan, M.~Boratko, Y.~Luan, S.~M.~R. Arnold, V.~Perot, S.~Dalmia, H.~Hu, X.~Lin, P.~Pasupat, A.~Amini, J.~R. Cole, S.~Riedel, I.~Naim, M.-W. Chang, and K.~Guu.
\newblock {Can Long-Context Language Models Subsume Retrieval, RAG, SQL, and More?}
\newblock \emph{arXiv preprint arXiv:2406.13121}, 2024.

\bibitem[Li et~al.(2023)Li, Wang, Ding, and Chen]{li2023large}
Y.~Li, S.~Wang, H.~Ding, and H.~Chen.
\newblock Large language models in finance: A survey.
\newblock In \emph{Proceedings of the 4th ACM International Conference on AI in Finance}, 2023.
\newblock \doi{10.1145/3604237.3626869}.

\bibitem[Li et~al.(2024)Li, Huang, Yang, Venkitesh, Locatelli, Ye, Cai, Lewis, and Chen]{li2024snapkv}
Y.~Li, Y.~Huang, B.~Yang, B.~Venkitesh, A.~Locatelli, H.~Ye, T.~Cai, P.~Lewis, and D.~Chen.
\newblock {SnapKV: LLM Knows What You are Looking for Before Generation}.
\newblock In \emph{The Thirty-eighth Annual Conference on Neural Information Processing Systems}, 2024.

\bibitem[Lin et~al.(2025)Lin, Zhang, Ge, Ren, Patra, Chaudhary, Peng, and Song]{lin2024s2attention}
X.~Lin, Y.~Zhang, S.~Ge, L.~Ren, B.~Patra, V.~Chaudhary, H.~Peng, and X.~Song.
\newblock S2-{ATTENTION}: {HARDWARE}-{AWARE} {CONTEXT} {SHARDING} {AMONG} {ATTENTION} {HEADS}.
\newblock In \emph{Sparsity in LLMs (SLLM): Deep Dive into Mixture of Experts, Quantization, Hardware, and Inference}, 2025.
\newblock URL \url{https://openreview.net/forum?id=AhlE339I7o}.

\bibitem[Liu et~al.(2026)Liu, Khandelwal, Subramanian, Jouault, Rastogi, Sadé, Jeffares, Jiang, Cahill, Gavaudan, Sablayrolles, Héliou, et~al.]{liu2026ministral3}
A.~H. Liu, K.~Khandelwal, S.~Subramanian, V.~Jouault, A.~Rastogi, A.~Sadé, A.~Jeffares, A.~Jiang, A.~Cahill, A.~Gavaudan, A.~Sablayrolles, A.~Héliou, et~al.
\newblock Ministral 3.
\newblock \emph{arXiv preprint arXiv:2601.08584}, 2026.

\bibitem[Liu et~al.(2025)Liu, Chen, Lu, Jiang, Han, Zhang, Chen, Zhang, Ding, Zhang, Chen, Yang, Yang, and Qiu]{liu2024retrievalattention}
D.~Liu, M.~Chen, B.~Lu, H.~Jiang, Z.~Han, Q.~Zhang, Q.~Chen, C.~Zhang, B.~Ding, K.~Zhang, C.~Chen, F.~Yang, Y.~Yang, and L.~Qiu.
\newblock {RetrievalAttention: Accelerating Long-Context LLM Inference via Vector Retrieval}.
\newblock In \emph{The Thirty-ninth Annual Conference on Neural Information Processing Systems}, 2025.

\bibitem[{Mathematical Association of America}(2025)]{dekoninck2026matharena}
{Mathematical Association of America}.
\newblock {American Invitational Mathematics Examination (AIME) 2025}.
\newblock \url{https://maa.org/maa-invitational-competitions/}, 2025.
\newblock Accessed: 2026-05-20.

\bibitem[Pagliardini et~al.(2023)Pagliardini, Paliotta, Jaggi, and Fleuret]{pagliardini2023fast}
M.~Pagliardini, D.~Paliotta, M.~Jaggi, and F.~Fleuret.
\newblock Fast attention over long sequences with dynamic sparse flash attention.
\newblock In \emph{Advances in Neural Information Processing Systems}, volume~36, pages~--, 2023.

\bibitem[Park et~al.(2025)Park, Jones, Morse, Goel, Lee, and Lott]{park2025keydiff}
J.~Park, D.~Jones, M.~J. Morse, R.~Goel, M.~Lee, and C.~Lott.
\newblock {KeyDiff: Key Similarity-Based {KV} Cache Eviction for Long-Context LLM Inference in Resource-Constrained Environments}.
\newblock In \emph{The Thirty-ninth Annual Conference on Neural Information Processing Systems}, 2025.

\bibitem[Peebles and Xie(2023)]{Peebles_2023_ICCV}
W.~Peebles and S.~Xie.
\newblock {Scalable Diffusion Models with Transformers}.
\newblock In \emph{Proceedings of the IEEE/CVF International Conference on Computer Vision (ICCV)}, 2023.

\bibitem[Pinelis(1994)]{pinelis1994optimum}
I.~Pinelis.
\newblock Optimum bounds for the distributions of martingales in banach spaces.
\newblock \emph{The Annals of Probability}, 22\penalty0 (4):\penalty0 1679--1706, 1994.
\newblock \doi{10.1214/aop/1176988477}.

\bibitem[Ruan and Jin(2022)]{ruan2021transformer}
L.~Ruan and Q.~Jin.
\newblock {Survey: Transformer based Video-Language Pre-training}.
\newblock \emph{AI Open}, 3:\penalty0 1--13, 2022.
\newblock \url{https://www.sciencedirect.com/science/article/pii/S2666651022000018}.

\bibitem[Sengupta et~al.(2025)Sengupta, Chaudhary, and Chakraborty]{sengupta2025valueguided}
A.~Sengupta, S.~Chaudhary, and T.~Chakraborty.
\newblock {Value-Guided KV Compression for LLMs via Approximated CUR Decomposition}.
\newblock In \emph{The Thirty-ninth Annual Conference on Neural Information Processing Systems}, 2025.

\bibitem[Shah et~al.(2024)Shah, Bikshandi, Zhang, Thakkar, Ramani, and Dao]{shah2024flashattention}
J.~Shah, G.~Bikshandi, Y.~Zhang, V.~Thakkar, P.~Ramani, and T.~Dao.
\newblock {FlashAttention-3: Fast and Accurate Attention with Asynchrony and Low-precision}.
\newblock In \emph{The Thirty-eighth Annual Conference on Neural Information Processing Systems}, 2024.

\bibitem[Shanbhag et~al.(2018)Shanbhag, Pirk, and Madden]{radixtopk}
A.~Shanbhag, H.~Pirk, and S.~Madden.
\newblock Efficient top-k query processing on massively parallel hardware.
\newblock In \emph{Proceedings of the 2018 International Conference on Management of Data}, SIGMOD '18. Association for Computing Machinery, 2018.
\newblock ISBN 9781450347037.
\newblock URL \url{https://doi.org/10.1145/3183713.3183735}.

\bibitem[Shazeer(2019)]{shazeer2019fast}
N.~Shazeer.
\newblock Fast transformer decoding: One write-head is all you need.
\newblock \emph{arXiv preprint arXiv:1911.02150}, 2019.

\bibitem[Singhania et~al.(2024)Singhania, Singh, He, Feizi, and Bhatele]{loki2024lowrank}
P.~Singhania, S.~Singh, S.~He, S.~Feizi, and A.~Bhatele.
\newblock {Loki: Low-rank Keys for Efficient Sparse Attention}.
\newblock In \emph{Advances in Neural Information Processing Systems}, 2024.

\bibitem[Tang et~al.(2024)Tang, Zhao, Zhu, Xiao, Kasikci, and Han]{pmlr-v235-tang24l}
J.~Tang, Y.~Zhao, K.~Zhu, G.~Xiao, B.~Kasikci, and S.~Han.
\newblock {QUEST: Query-Aware Sparsity for Efficient Long-Context LLM Inference}.
\newblock In \emph{International Conference on Machine Learning}, 2024.

\bibitem[Team(2026)]{qwen35blog}
Q.~Team.
\newblock Qwen3.5: Accelerating productivity with native multimodal agents, February 2026.
\newblock URL \url{https://qwen.ai/blog?id=qwen3.5}.

\bibitem[Tokdar and Kass(2010)]{tokdar2010importance}
S.~T. Tokdar and R.~E. Kass.
\newblock Importance sampling: a review.
\newblock \emph{Wiley Interdisciplinary Reviews: Computational Statistics}, 2, 2010.

\bibitem[Vitter(1985)]{jeffreyuniform}
J.~S. Vitter.
\newblock Random sampling with a reservoir.
\newblock \emph{ACM Transactions on Mathematical Software (TOMS)}, 11\penalty0 (1), 1985.

\bibitem[Yan et~al.(2025)Yan, Jiang, Zhang, Ma, Zhu, Cao, and Xu]{zhang2024adamas}
S.~Yan, G.-Q. Jiang, Y.~Zhang, X.~Ma, R.~Zhu, C.~Cao, and J.~Xu.
\newblock {Adamas: Hadamard Sparse Attention for Efficient Long-Context Inference}.
\newblock \emph{arXiv preprint arXiv:2510.18413}, 2025.

\bibitem[Yang et~al.(2026)Yang, Tang, Zhen, and Lv]{yang2026rethinking}
J.~Yang, C.~Tang, L.~Zhen, and J.~Lv.
\newblock Rethinking {KV} cache eviction via a unified information-theoretic objective.
\newblock In \emph{Forty-third International Conference on Machine Learning}, 2026.

\bibitem[Yang et~al.(2024)Yang, Sheng, Gonzalez, Stoica, and Zheng]{yang2024double}
S.~Yang, Y.~Sheng, J.~E. Gonzalez, I.~Stoica, and L.~Zheng.
\newblock {Post-Training Sparse Attention with Double Sparsity}.
\newblock \emph{arXiv preprint arXiv:2408.07092}, 2024.

\bibitem[Zhang et~al.(2025)Zhang, Ji, Chen, Fu, Miao, Nie, Chen, and Cui]{pqcache2025}
H.~Zhang, X.~Ji, Y.~Chen, F.~Fu, X.~Miao, X.~Nie, W.~Chen, and B.~Cui.
\newblock {PQCache: Product Quantization-based KVCache for Long Context LLM Inference}.
\newblock \emph{Proceedings of the ACM on Management of Data}, 3\penalty0 (3):\penalty0 1--30, 2025.

\bibitem[Zhang et~al.(2023)Zhang, Sheng, Zhou, Chen, Zheng, Cai, Song, Tian, R{\'e}, Barrett, Wang, and Chen]{zhang2023h2o}
Z.~Zhang, Y.~Sheng, T.~Zhou, T.~Chen, L.~Zheng, R.~Cai, Z.~Song, Y.~Tian, C.~R{\'e}, C.~W. Barrett, Z.~Wang, and B.~Chen.
\newblock {H$_2$O: Heavy-Hitter Oracle for Efficient Generative Inference of Large Language Models}.
\newblock In \emph{Advances in Neural Information Processing Systems}, 2023.

\bibitem[Zhou et~al.(2024{\natexlab{a}})Zhou, Sun, and Li]{Zhou2024DBGPTLL}
X.~Zhou, Z.~Sun, and G.~Li.
\newblock {DB-GPT: Large Language Model Meets Database}.
\newblock \emph{Data Science and Engineering}, 9:\penalty0 102 -- 111, 2024{\natexlab{a}}.

\bibitem[Zhou et~al.(2024{\natexlab{b}})Zhou, Wang, Zeng, Guo, Liu, Shen, Zhang, and Ding]{zhou2025dynamickv}
X.~Zhou, W.~Wang, M.~Zeng, J.~Guo, X.~Liu, L.~Shen, M.~Zhang, and L.~Ding.
\newblock {DynamicKV: Task-Aware Adaptive KV Cache Compression for Long Context LLMs}.
\newblock \emph{arXiv preprint arXiv:2412.14838}, 2024{\natexlab{b}}.

\bibitem[Zhu et~al.(2026)Zhu, Tang, Xu, Jin, Gu, Zeng, Kadekodi, Zhao, Li, Krishnamurthy, and Kasikci]{zhu2025topp}
K.~Zhu, T.~Tang, Q.~Xu, Z.~Jin, Y.~Gu, Z.~Zeng, R.~Kadekodi, L.~Zhao, A.~Li, A.~Krishnamurthy, and B.~Kasikci.
\newblock {Tactic: Adaptive Sparse Attention with Clustering and Distribution Fitting for Long-Context LLMs}.
\newblock In \emph{The Fourteenth International Conference on Learning Representations}, 2026.

\end{thebibliography}
}

\newpage
\appendix

\section{Appendix}

\subsection{Proof of Lemma~\ref{lem:deduplication}}
\begin{proof}
Token $i$ is included if it appears in at least one of the $m$
independent draws, so
\[
\pi_i=1-(1-p_i)^m.
\]
By the union bound,
\[
\pi_i
=
\Pr\!\left(\bigcup_{\ell=1}^m\{J_\ell=i\}\right)
\le
\sum_{\ell=1}^m\Pr(J_\ell=i)
=
mp_i.
\]

For the lower bound, $1-p_i\le e^{-p_i}$ gives
\[
\pi_i
=
1-(1-p_i)^m
\ge
1-e^{-mp_i}.
\]
The function $g(x)=(1-e^{-x})/x$ is decreasing for $x>0$.
Since $mp_i\le c$,
\[
1-e^{-mp_i}
=
mp_i\,g(mp_i)
\ge
mp_i\,g(c)
=
\kappa(c)mp_i.
\]
This proves the two-sided bound on $\pi_i$.
\\
\noindent
Finally, the definition of $\rho$ gives
$a_i=A_{\mathcal R}\bar a_i\le A_{\mathcal R}\rho p_i$.
Combining this with the lower bound on $\pi_i$ yields
\[
\frac{a_i}{\pi_i}
\le
\frac{A_{\mathcal R}\rho p_i}
     {\kappa(c)mp_i}
=
\frac{A_{\mathcal R}\rho}{\kappa(c)m}.
\]
\end{proof}

\subsection{Proof of Theorem~\ref{thm:ssa-hoeffding}}
\label{app:proof-ssa-hoeffding}

We first establish the following dimension-free bounded-differences inequality
used in the proof.

\begin{lemma}[Vector-valued bounded differences]
\label{lem:hilbert-bounded-differences}
Let $X_1,\ldots,X_m$ be independent random variables, and let
$f$ be a measurable function taking values in $\mathbb R^r$.
Suppose that changing only the $\ell^{th}$ argument changes $f$
by at most $b_\ell$ in Euclidean norm \ie,
\[
\left\|
f(x_1,\ldots,x_\ell,\ldots,x_m)
-
f(x_1,\ldots,x_\ell',\ldots,x_m)
\right\|_2
\le b_\ell
\]
for every $\ell\in[m]$ and every pair of input sequences
differing only in coordinate $\ell$, where $b_\ell\ge0$ and
$\sum_{\ell=1}^m b_\ell^2>0$. Then, for every $t>0$,
\begin{equation}
\Pr\left(
\left\|
f(X_1,\ldots,X_m)
-
\mathbb E[f(X_1,\ldots,X_m)]
\right\|_2
\ge t
\right)
\le
2\exp\left(
-\frac{t^2}{2\sum_{\ell=1}^m b_\ell^2}
\right).
\label{eq:hilbert-bounded-differences}
\end{equation}
\end{lemma}

\begin{proof}
Reveal the variables one at a time and define
\[
M_\ell
=
\mathbb E\!\left[
f(X_1,\ldots,X_m)\mid X_1,\ldots,X_\ell
\right],
\qquad \ell=0,\ldots,m.
\]
Here
\[
M_0=\mathbb E[f(X_1,\ldots,X_m)],
\qquad
M_m=f(X_1,\ldots,X_m).
\]
The sequence $(M_\ell)_{\ell=0}^m$ is a vector-valued
martingale.
\\
\noindent
We first bound its increments. Fix the first $\ell-1$
variables and let
\[
g_\ell(x)
=
\mathbb E_{X_{\ell+1},\ldots,X_m}\!\left[
f(X_1,\ldots,X_{\ell-1},x,X_{\ell+1},\ldots,X_m)
\right].
\]
Averaging over the same future variables preserves the
bounded-differences condition:
\[
\|g_\ell(x)-g_\ell(x')\|_2\le b_\ell.
\]
Let $X_\ell'$ be an independent copy of $X_\ell$, independent
of all the original variables. Since
$M_\ell=g_\ell(X_\ell)$ and $M_{\ell-1}$ averages
$g_\ell$ over the $\ell^{th}$ variable, Jensen's inequality gives
\begin{align*}
\|M_\ell-M_{\ell-1}\|_2
&=
\left\|
\mathbb E\!\left[
g_\ell(X_\ell)-g_\ell(X_\ell')
\mid X_1,\ldots,X_\ell
\right]
\right\|_2\\
&\le
\mathbb E\!\left[
\|g_\ell(X_\ell)-g_\ell(X_\ell')\|_2
\mid X_1,\ldots,X_\ell
\right]\\
&\le b_\ell.
\end{align*}
\noindent
The Euclidean specialization of the martingale inequality
in \cite[Theorem~3.5]{pinelis1994optimum}, applied to
$(M_\ell-M_0)_{\ell=0}^m$, therefore yields
\[
\Pr\left(
\|M_m-M_0\|_2\ge t
\right)
\le
2\exp\left(
-\frac{t^2}{2\sum_{\ell=1}^m b_\ell^2}
\right).
\]
Substituting $M_m$ and $M_0$ proves the claim.
\end{proof}

Now, we are ready to prove Theorem~\ref{thm:ssa-hoeffding}.
\begin{proof}[Proof of Theorem~\ref{thm:ssa-hoeffding}]
We consider the sparse regime with an integer budget $1\le k<N$
and finite attention and proxy scores. The case $k=N$ can be
handled by evaluating dense attention exactly. Recall that $\mathcal H$ and $\mathcal R=[N]\setminus\mathcal H$ be the
heavy and residual sets produced by SSA, and let
$h=|\mathcal H|$ and $m=k-h$. As shown in Section~\ref{sec:theory},
$\mathcal R\neq\emptyset$ and $m\ge1$. Since the proxy scores are finite,
$p_i>0$ for every $i\in\mathcal R$.
\\
\noindent
Let $J_1,\ldots,J_m$ denote the independent residual draws:
\[
\Pr(J_\ell=i)=p_i,
\qquad i\in\mathcal R,
\qquad \ell\in[m].
\]
For each residual token, define its occupancy indicator
\[
Z_i
=
\mathbf 1\!\left\{
J_\ell=i\text{ for some }\ell\in[m]
\right\},
\qquad i\in\mathcal R.
\]
We know that its inclusion probability is given by
$\pi_i
=
\Pr(Z_i=1)
=
1-(1-p_i)^m$,
so $\mathbb E[Z_i]=\pi_i$. Heavy tokens are retained
deterministically and have inclusion probability one.
\\
\noindent
Note that although the residual indicators $\{Z_i\}_{i\in\mathcal R}$ are
dependent, they are deterministic functions of the independent
draws $J_1,\ldots,J_m$. We therefore apply
Lemma~\ref{lem:hilbert-bounded-differences} to these draws.
\\
\noindent
Write
\[
D:=\sum_{i=1}^{N}w_i,
\qquad
a_i=\frac{w_i}{D},
\qquad
\mathbf o=\sum_{i=1}^{N}a_i\mathbf v_i.
\]
Define the estimated denominator and numerator by
\begin{align*}
\widehat D
&:=
\sum_{i\in\mathcal H}w_i
+
\sum_{i\in\mathcal R}\frac{Z_iw_i}{\pi_i},\\
\widehat{\mathbf N}
&:=
\sum_{i\in\mathcal H}w_i\mathbf v_i
+
\sum_{i\in\mathcal R}
\frac{Z_iw_i}{\pi_i}\mathbf v_i.
\end{align*}
The SSA output is therefore
\[
\widehat{\mathbf o}
=
\frac{\widehat{\mathbf N}}{\widehat D}.
\]
Since $m\ge1$ and all weights are positive,
$\widehat D>0$ for every realization of the residual draws.
\\
\noindent
If $\|\mathbf V\|_2=0$, all value vectors vanish,
and $\widehat{\mathbf o}=\mathbf o=\mathbf 0$, so the theorem is
immediate. Therefore, we assume $\|\mathbf V\|_2>0$.
\\
\noindent
By Lemma~\ref{lem:deduplication},
\[
\frac{a_i}{\pi_i}
\le
\frac{A_{\mathcal R}\rho}{\kappa(c)m},
\qquad i\in\mathcal R.
\]
For brevity, define
\begin{equation}
\beta
:=
\frac{A_{\mathcal R}\rho}{\kappa(c)m}.
\label{eq:beta-corrected-weight}
\end{equation}
Thus $a_i/\pi_i\le\beta$ for every $i\in\mathcal R$.
\\
\noindent
Next, we define
\begin{equation}
\Delta_D
:=
\frac{\widehat D-D}{D}.
\label{eq:normalized-denominator-error}
\end{equation}
The contributions from $\mathcal H$ are evaluated exactly and
cancel in $\widehat D-D$. Consequently,
\begin{equation}
\Delta_D
=
\sum_{i\in\mathcal R}
a_i\left(\frac{Z_i}{\pi_i}-1\right).
\label{eq:denominator-error-expansion}
\end{equation}
Since $\mathbb E[Z_i]=\pi_i$,
\[
\mathbb E[\Delta_D]=0.
\]
\noindent
Consider two draw sequences $\mathbf J=(J_1,\ldots,J_m)$ and
$\mathbf J'$ differing only in coordinate $\ell$, where the draw
changes from $r$ to $s$. Only $Z_r$ and $Z_s$ can change, and each
changes by at most one. Therefore,
\begin{align}
\left|
\Delta_D(\mathbf J)-\Delta_D(\mathbf J')
\right|
&\le
\frac{a_r}{\pi_r}+\frac{a_s}{\pi_s}
\notag\\
&\le 2\beta.
\label{eq:denominator-sensitivity}
\end{align}
\noindent
Applying Lemma~\ref{lem:hilbert-bounded-differences} with
$r=1$ and $b_\ell=2\beta$ gives, for every $t>0$,
\[
\Pr\left(|\Delta_D|\ge t\right)
\le
2\exp\left(-\frac{t^2}{8m\beta^2}\right).
\]
Taking $t=2\sqrt{2m}\,\beta\sqrt u$, we obtain, for every $u>0$,
\begin{equation}
\Pr\left(
|\Delta_D|
\ge
2\sqrt{2m}\,\beta\sqrt u
\right)
\le 2e^{-u}.
\label{eq:denominator-concentration-intermediate}
\end{equation}
Substituting eq.~\ref{eq:beta-corrected-weight} yields
\begin{equation}
\Pr\left(
|\Delta_D|
\ge
\frac{2\sqrt{2}\,A_{\mathcal R}\rho}{\kappa(c)}
\sqrt{\frac{u}{m}}
\right)
\le 2e^{-u}.
\label{eq:denominator-concentration}
\end{equation}
\noindent
Now, we define
\begin{equation}
\mathbf E
:=
\frac{\widehat{\mathbf N}-\mathbf o\widehat D}{D}.
\label{eq:centered-vector-error}
\end{equation}
Because $\sum_{i=1}^{N}a_i=1$ and
$\mathbf o=\sum_{i=1}^{N}a_i\mathbf v_i$, we have
\[
\sum_{i=1}^{N}a_i(\mathbf v_i-\mathbf o)=\mathbf 0.
\]
It follows that
\begin{align}
\mathbf E
&=
\sum_{i\in\mathcal H}a_i(\mathbf v_i-\mathbf o)
+
\sum_{i\in\mathcal R}
\frac{Z_i a_i}{\pi_i}(\mathbf v_i-\mathbf o)
\notag\\
&=
\sum_{i\in\mathcal R}
a_i\left(\frac{Z_i}{\pi_i}-1\right)
(\mathbf v_i-\mathbf o).
\label{eq:centered-vector-expansion}
\end{align}
Again using $\mathbb E[Z_i]=\pi_i$, we obtain
\[
\mathbb E[\mathbf E]=\mathbf 0.
\]
\noindent
Let $\mathbf e_i$ denote the $i^{th}$ standard basis vector of
$\mathbb R^N$. Since the value vectors are the rows of
$\mathbf V$, for every $i\in[N]$,
\[
\|\mathbf v_i\|_2
=
\|\mathbf V^\top\mathbf e_i\|_2
\le
\|\mathbf V^\top\|_2\|\mathbf e_i\|_2
=
\|\mathbf V\|_2.
\]
Moreover, $\mathbf o$ is a convex combination of the value
vectors, so
\[
\|\mathbf o\|_2
\le
\sum_{i=1}^{N}a_i\|\mathbf v_i\|_2
\le
\|\mathbf V\|_2.
\]
Therefore,
\begin{equation}
\|\mathbf v_i-\mathbf o\|_2
\le
\|\mathbf v_i\|_2+\|\mathbf o\|_2
\le
2\|\mathbf V\|_2.
\label{eq:value-deviation-bound}
\end{equation}
\noindent
Now consider draw sequences $\mathbf J$ and $\mathbf J'$ that
differ only in one draw, changed from $r$ to $s$. As before, only
$Z_r$ and $Z_s$ can change. Hence,
\begin{align}
\|\mathbf E(\mathbf J)-\mathbf E(\mathbf J')\|_2
&\le
\frac{a_r}{\pi_r}\|\mathbf v_r-\mathbf o\|_2
+
\frac{a_s}{\pi_s}\|\mathbf v_s-\mathbf o\|_2
\notag\\
&\le
4\beta\|\mathbf V\|_2.
\label{eq:vector-sensitivity}
\end{align}
\noindent
Applying Lemma~\ref{lem:hilbert-bounded-differences} with
$r=d$ and $b_\ell=4\beta\|\mathbf V\|_2$ gives,
for every $t>0$,
\[
\Pr\left(\|\mathbf E\|_2\ge t\right)
\le
2\exp\left(
-\frac{t^2}{32m\beta^2\|\mathbf V\|_2^2}
\right).
\]
Taking
$t=4\sqrt{2m}\,\beta\sqrt u\,\|\mathbf V\|_2$ gives, for
every $u>0$,
\begin{equation}
\Pr\left(
\|\mathbf E\|_2
\ge
4\sqrt{2m}\,\beta\sqrt u\,\|\mathbf V\|_2
\right)
\le 2e^{-u}.
\label{eq:vector-concentration-intermediate}
\end{equation}
Equivalently,
\begin{equation}
\Pr\left(
\|\mathbf E\|_2
\ge
\frac{4\sqrt{2}\,A_{\mathcal R}\rho}{\kappa(c)}
\sqrt{\frac{u}{m}}\,
\|\mathbf V\|_2
\right)
\le 2e^{-u}.
\label{eq:vector-concentration}
\end{equation}
\noindent
Now, fix the failure probability $\delta\in(0,1)$ and set
\[
u=\log\frac{4}{\delta}.
\]
Then $2e^{-u}=\delta/2$. By
eq.~\ref{eq:denominator-concentration}, with probability at least
$1-\delta/2$,
\[
|\Delta_D|
\le
\frac{2\sqrt{2}\,A_{\mathcal R}\rho}{\kappa(c)}
\sqrt{\frac{\log(4/\delta)}{m}}
=
\epsilon_D(\delta).
\]
Similarly, by eq.~\ref{eq:vector-concentration}, with probability
at least $1-\delta/2$,
\[
\|\mathbf E\|_2
\le
\frac{4\sqrt{2}\,A_{\mathcal R}\rho}{\kappa(c)}
\sqrt{\frac{\log(4/\delta)}{m}}\,
\|\mathbf V\|_2.
\]
A union bound shows that both inequalities hold simultaneously
with probability at least $1-\delta$.
\\
\noindent
By the definition of $\Delta_D$,
\[
\frac{\widehat D}{D}=1+\Delta_D.
\]
Consequently,
\begin{align}
\widehat{\mathbf o}-\mathbf o
=
\frac{\widehat{\mathbf N}}{\widehat D}-\mathbf o
=
\frac{\widehat{\mathbf N}-\mathbf o\widehat D}
     {\widehat D}
=
\frac{\mathbf E}{1+\Delta_D}.
\label{eq:ratio-error-identity}
\end{align}
\noindent
On the simultaneous concentration event, the assumption
$\epsilon_D(\delta)<1$ implies
\[
1+\Delta_D
\ge
1-|\Delta_D|
\ge
1-\epsilon_D(\delta)
>0.
\]
Therefore,
\begin{align*}
\|\widehat{\mathbf o}-\mathbf o\|_2
&=
\frac{\|\mathbf E\|_2}{1+\Delta_D}\le
\frac{\|\mathbf E\|_2}{1-\epsilon_D(\delta)}\\
&\le
\frac{
4\sqrt{2}\,A_{\mathcal R}\rho
}{
\kappa(c)\bigl(1-\epsilon_D(\delta)\bigr)
}
\sqrt{\frac{\log(4/\delta)}{m}}\,
\|\mathbf V\|_2.
\end{align*}
This proves eq.~\ref{eq:ssa-hoeffding} with probability at least
$1-\delta$.
\end{proof}

\subsection{Proof of Corollary~\ref{cor:ssa-sample-complexity}}

\begin{proof}
Let
\[
x:=\epsilon_D(\delta)
=
\frac{2\sqrt{2}\,A_{\mathcal R}\rho}{\kappa(c)}
\sqrt{\frac{\log(4/\delta)}{m}}.
\]
The assumed lower bound on $m$ gives
\[
x
\le
\frac{2\sqrt{2}}{\sqrt{21}\,\kappa(c)}
\frac{\varepsilon}{2+\varepsilon}.
\]
Under the Algorithm~\ref{alg:ssa} allocation rule, $c<1$, so
$\kappa(c)>1-e^{-1}$. Since
\[
\frac{8}{(1-e^{-1})^2}<21,
\]
we have
\[
x\le\frac{\varepsilon}{2+\varepsilon}<1.
\]
Thus the denominator-stability condition in
Theorem~\ref{thm:ssa-hoeffding} holds.
\\
\noindent
The numerator of the bound in Theorem~\ref{thm:ssa-hoeffding}
equals $2x\|\mathbf V\|_2$. Therefore, with probability at
least $1-\delta$,
\[
\|\widehat{\mathbf o}-\mathbf o\|_2
\le
\frac{2x}{1-x}\|\mathbf V\|_2
\le
\frac{2\varepsilon/(2+\varepsilon)}
     {1-\varepsilon/(2+\varepsilon)}
\|\mathbf V\|_2
=
\varepsilon\|\mathbf V\|_2.
\]
\end{proof}

\newpage
\subsection{Additional Experiments}
\begin{figure}[H]
    \centering

    \begin{minipage}[t]{0.48\linewidth}
        \centering
        \includegraphics[width=\linewidth]{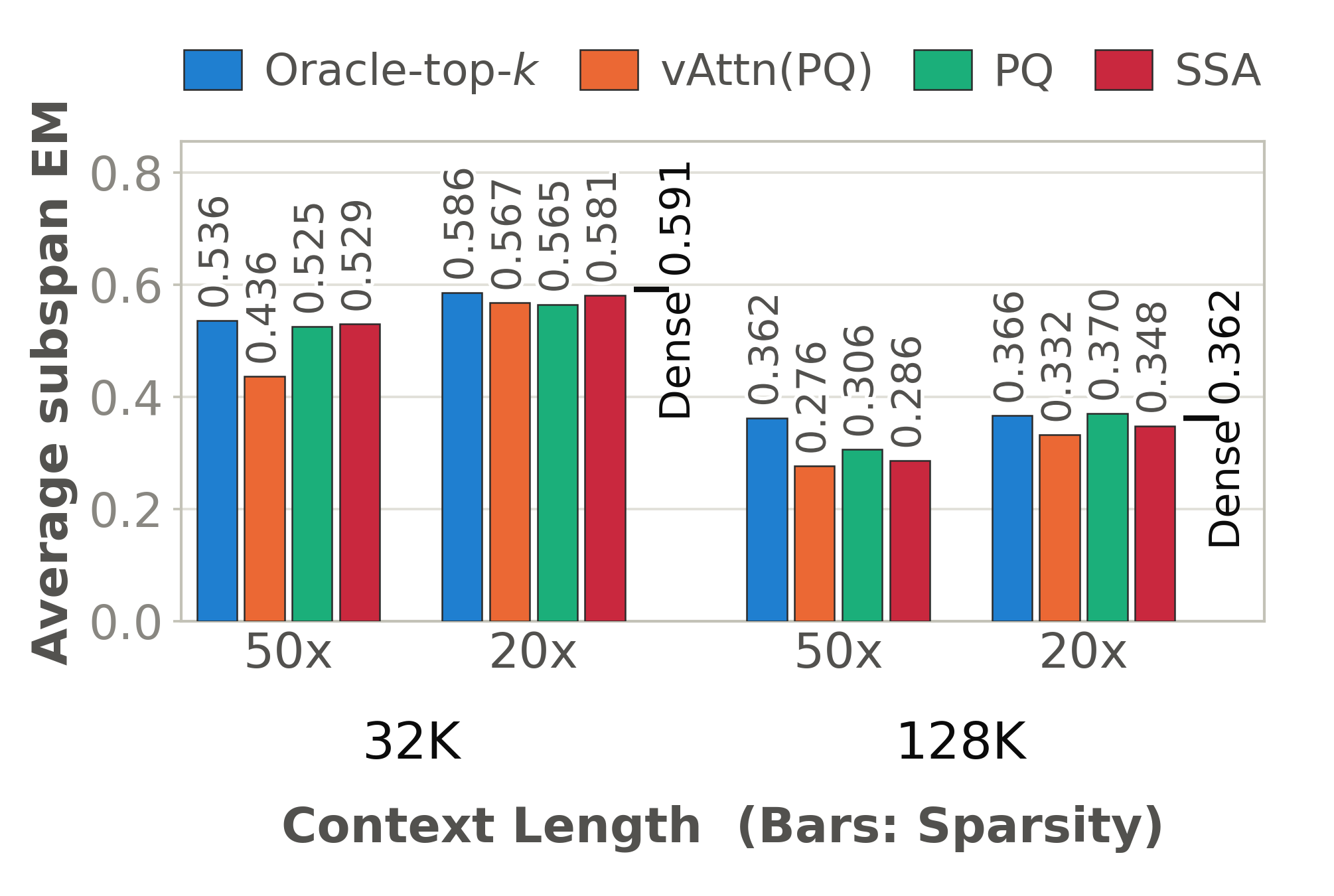}
        \vspace{-0.8em}

        {(a) Ministral-3-8B: Subspan-EM}
    \end{minipage}
    \hfill
    \begin{minipage}[t]{0.48\linewidth}
        \centering
        \includegraphics[width=\linewidth]{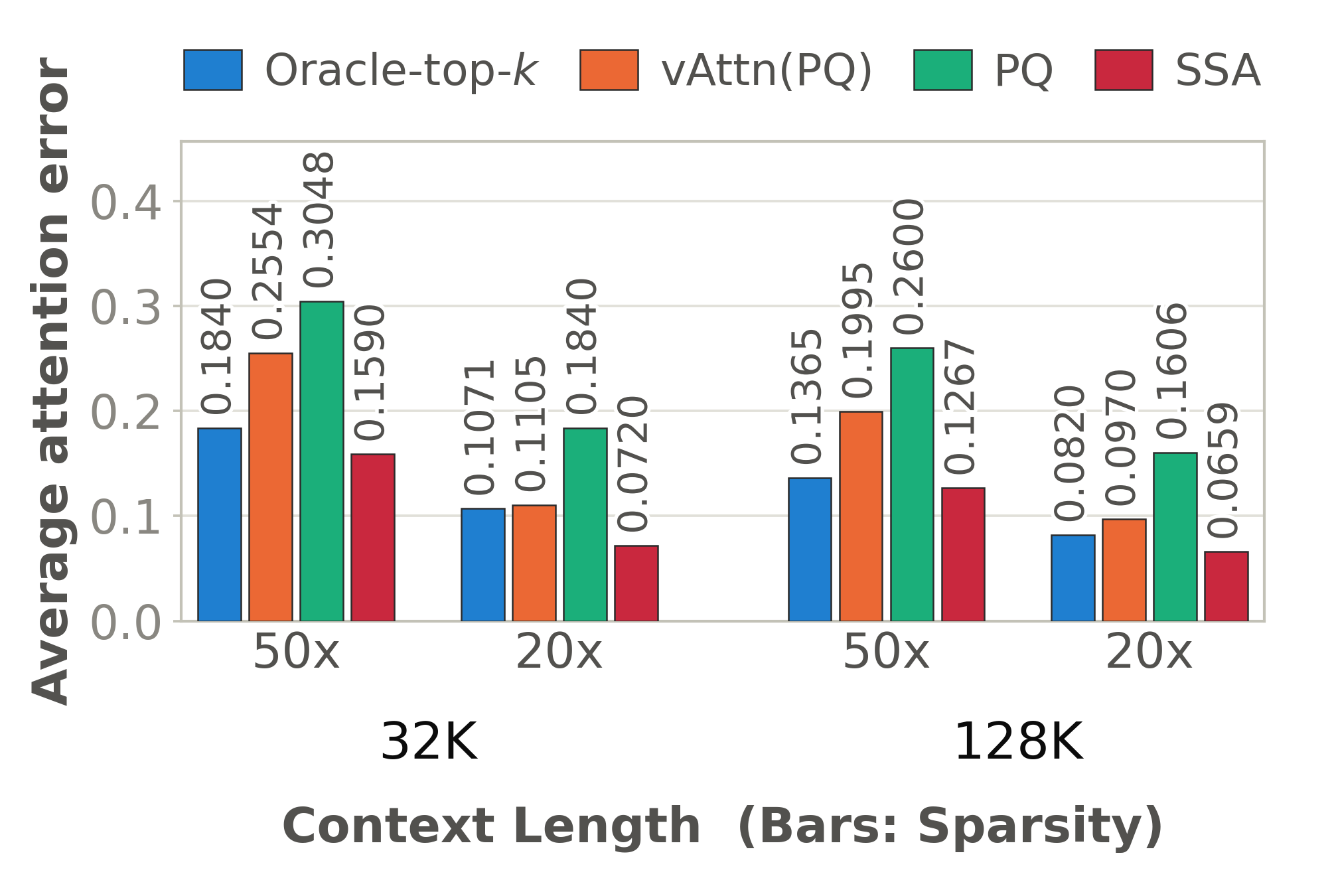}
        \vspace{-0.8em}

        {(b) Ministral-3-8B: Error}
    \end{minipage}

    \vspace{0.8em}

    \begin{minipage}[t]{0.48\linewidth}
        \centering
        \includegraphics[width=\linewidth]{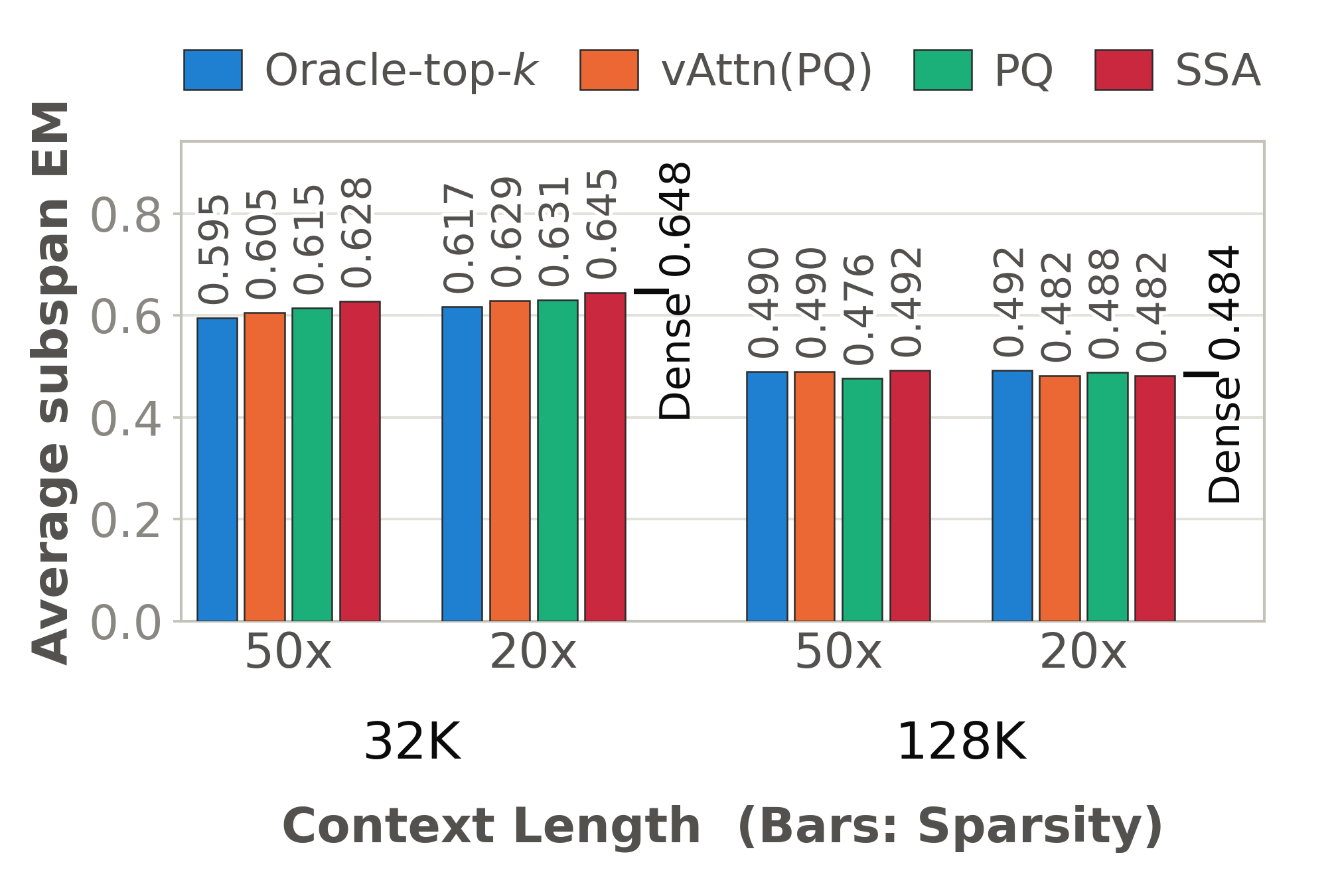}
        \vspace{-0.8em}

        {(c) Qwen3.5-4B: Subspan-EM}
    \end{minipage}
    \hfill
    \begin{minipage}[t]{0.48\linewidth}
        \centering
        \includegraphics[width=\linewidth]{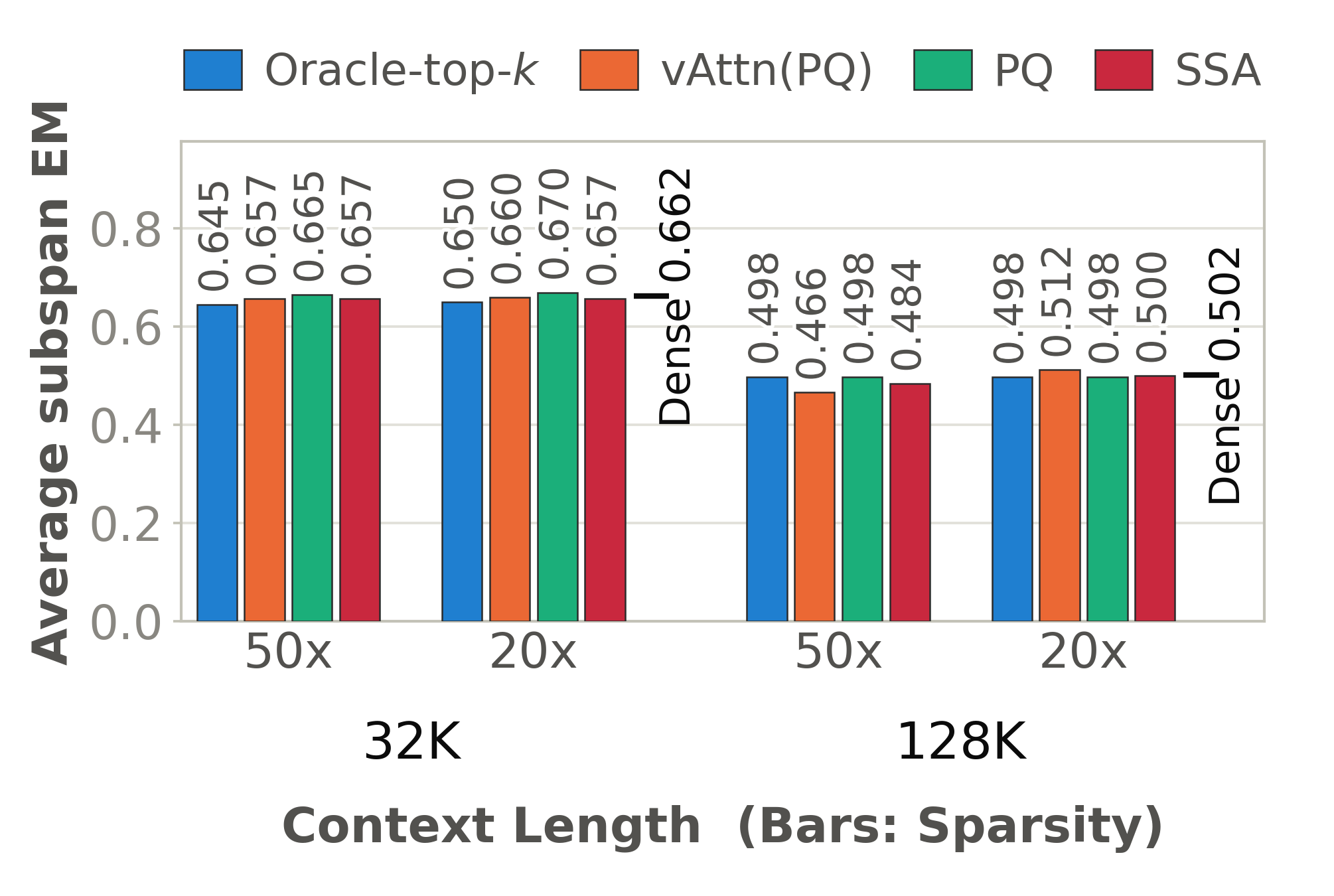}
        \vspace{-0.8em}

        {(d) Qwen3.5-9B: Subspan-EM}
    \end{minipage}

    \caption{Average Subspan-EM and error across different context lengths and sparsity levels on LOFT.}
    \label{fig:loft_results}
\end{figure}

\begin{table}[H]
\centering
\caption{Comparison of theoretical and practical SSA variants on average RULER-HARD-32K performance.}
\label{tab:ssa_theory_practical}
\begin{tabular}{lccc}
\toprule
\textbf{Model} & \textbf{Sparsity} & \textbf{SSA-Theory} & \textbf{SSA-Practical} \\
\midrule
Llama-3.1-8B   & 50$\times$ & 81 & 83 \\
Llama-3.1-8B   & 20$\times$ & 88 & 88 \\
Ministral-3-8B & 50$\times$ & 82 & 84 \\
Ministral-3-8B & 20$\times$ & 89 & 88 \\
\bottomrule
\end{tabular}
\end{table}

\begin{table}[H]
    \centering
    \caption{Average performance on RULER-HARD-32K across different models at 50$\times$ sparsity.}
    \label{tab:ruler_hard_32k_all_models}
    \setlength{\tabcolsep}{6pt}
    \resizebox{\linewidth}{!}{%
    \begin{tabular}{lccccccc}
        \toprule
        \textbf{Model}
        & \textbf{Dense}
        & \textbf{PQCache}
        & \textbf{vAttn(PQ)}
        & \textbf{Quest}
        & \textbf{HashAttn}
        & \textbf{Double Sparsity}
        & \textbf{MagicPig} \\
        \midrule

        Qwen3-4B-Instruct-2507
        & 88.67
        & \textbf{85.80}
        & 85.67
        & 42.03
        & --
        & 39.29
        & 18.72 \\

        Llama-3.2-3B-Instruct
        & 66.25
        & 39.30
        & \textbf{39.59}
        & 34.11
        & --
        & 31.21
        & 17.22 \\

        Llama-3.1-8B-Instruct
        & 87.89
        & \textbf{68.95}
        & 64.42
        & 64.17
        & 49.48
        & 44.62
        & 16.28 \\

        \bottomrule
    \end{tabular}%
    }
\end{table}

\begin{table}[H]
\centering
\caption{Evaluation datasets and subsets used in our experiments.}
\label{tab:datasets_info}
\begin{tabularx}{\linewidth}{l c X}
\toprule
\textbf{Dataset} & \textbf{\# Samples} & \textbf{Subsets} \\
\midrule

RULER-HARD
& 100
& fwe, qa1, qa2, niah\_multikey\_2, niah\_multikey\_3, vt \\

\addlinespace

LongBench
& 200
& NrtvQA, Qasper, MF-en, HotpotQA, 2WikiMQA, Musique, DuReader, GovRep, QMSum, MultiNews, VCSum, TREC, TriviaQA, SAMSum, PCount, PR-en, LCC, RBP \\

\addlinespace

LOFT
& 100
& HotpotQA, Musique, NQ, Qampari, Quest \\

\addlinespace

AIME
& 16
& -- \\

\bottomrule
\end{tabularx}
\end{table}

\begin{table}[H]
\centering
\caption{Performance on RULER-HARD with Llama-3.1-8B.}
\label{tab:ruler_llama}
\setlength{\tabcolsep}{5pt}
\begin{tabular}{cclccccccc}
\toprule
\textbf{Sparsity} & \textbf{Context} & \textbf{Method} 
& \textbf{FWE} & \textbf{VT} & \textbf{NM2} & \textbf{NM3} 
& \textbf{QA1} & \textbf{QA2} & \textbf{Avg.} \\
\midrule

\multirow{15}{*}{\textbf{50$\times$}}
& \multirow{5}{*}{32K}
& Dense          & 93.33 & 99.2 & 100 & 100 & 88 & 52 & 88.76 \\
& & PQ      & 77.67 & 95.4 & 67  & 71  & 76 & 48 & 72.51 \\
& & vAttn (PQ) & 83.67 & 94.6 & 52 & 58 & 66 & 46 & 66.71 \\
& & Oracle-top-$k$       & 81.00 & 97.8 & 87  & 76  & 79 & 52 & 78.80 \\
& & SSA       & 91.00 & 98.0 & 83  & 92  & 85 & 53 & \textbf{83.67} \\
\cmidrule(lr){2-10}

& \multirow{5}{*}{64K}
& Dense          & 86.33 & 98.8 & 98 & 99 & 84 & 53 & 86.52 \\
& & PQ      & 80.00 & 92.2 & 65 & 50 & 71 & 47 & 67.53 \\
& & vAttn (PQ) & 83.33 & 94.0 & 40 & 68 & 67 & 45 & 66.22 \\
& & Oracle-top-$k$       & 79.67 & 95.8 & 78 & 56 & 77 & 45 & 71.91 \\
& & SSA       & 85.33 & 96.2 & 67 & 84 & 77 & 53 & \textbf{77.09} \\
\cmidrule(lr){2-10}

& \multirow{5}{*}{128K}
& Dense          & 73.33 & 66.2 & 91 & 67 & 83 & 47 & 71.26 \\
& & PQ      & 66.00 & 39.4 & 27 & 15 & 57 & 34 & 39.73 \\
& & vAttn (PQ) & 71.67 & 55.0 & 21 & 20 & 41 & 28 & 39.45 \\
& & Oracle-top-$k$       & 61.00 & 66.6 & 71 & 32 & 66 & 40 & \textbf{56.10} \\
& & SSA       & 73.00 & 54.2 & 30 & 30 & 59 & 39 & 47.53 \\

\midrule
\midrule

\multirow{12}{*}{\textbf{20$\times$}}
& \multirow{4}{*}{32K}
& PQ        & 87.33 & 98.6 & 97  & 93  & 84 & 52 & 85.32 \\
& & vAttn (PQ) & 93.33 & 99.4 & 97 & 99 & 84 & 52 & 87.46 \\
& & Oracle-top-$k$       & 90.33 & 98.8 & 100 & 99 & 84 & 52 & 87.36 \\
& & SSA       & 93.67 & 99.0 & 97  & 100 & 87 & 54 & \textbf{88.45} \\
\cmidrule(lr){2-10}

& \multirow{4}{*}{64K}
& PQ        & 86.33 & 96.2 & 91 & 91 & 77 & 55 & 82.76 \\
& & vAttn (PQ) & 86.00 & 98.4 & 87 & 98 & 79 & 57 & 84.23 \\
& & Oracle-top-$k$       & 85.33 & 98.0 & 94 & 95 & 81 & 52 & 84.22 \\
& & SSA       & 85.33 & 98.6 & 94 & 99 & 81 & 54 & \textbf{85.32} \\
\cmidrule(lr){2-10}

& \multirow{4}{*}{128K}
& PQ        & 69.00 & 65.8 & 66 & 46 & 68 & 34 & 58.13 \\
& & vAttn (PQ) & 73.33 & 63.6 & 61 & 51 & 69 & 39 & 59.49 \\
& & Oracle-top-$k$       & 67.00 & 63.2 & 85 & 51 & 72 & 42 & 63.37 \\
& & SSA       & 75.00 & 69.0 & 70 & 63 & 76 & 44 & \textbf{66.17} \\
\bottomrule
\end{tabular}
\end{table}
\begin{table}[H]
\centering
\caption{Performance on RULER-HARD with Ministral-3-3B.}
\label{tab:ruler_ministral3b}
\setlength{\tabcolsep}{5pt}
\begin{tabular}{cclccccccc}
\toprule
\textbf{Sparsity} & \textbf{Context} & \textbf{Method}
& \textbf{FWE} & \textbf{VT} & \textbf{NM2} & \textbf{NM3}
& \textbf{QA1} & \textbf{QA2} & \textbf{Avg.} \\
\midrule

\multirow{15}{*}{\textbf{50$\times$}}
& \multirow{5}{*}{32K}
& Dense             & 96.33 & 91.8 & 98 & 99 & 79 & 56 & 86.69 \\
& & PQ         & 92.67 & 60.2 & 91 & 56 & 73 & 52 & 70.81 \\
& & vAttn (PQ) & 90.33 & 72.2 & 67 & 77 & 65 & 44 & 69.26 \\
& & Oracle-top-$k$          & 93.67 & 62.2 & 99 & 55 & 75 & 52 & 72.81 \\
& & SSA          & 95.67 & 88.6 & 92 & 90 & 72 & 54 & \textbf{82.05} \\
\cmidrule(lr){2-10}

& \multirow{5}{*}{64K}
& Dense             & 74.33 & 99.8 & 93 & 96 & 75 & 55 & 82.19 \\
& & PQ         & 72.67 & 94.6 & 77 & 51 & 69 & 55 & 69.88 \\
& & vAttn (PQ) & 76.00 & 95.6 & 54 & 75 & 52 & 42 & 65.77 \\
& & Oracle-top-$k$          & 75.33 & 98.2 & 89 & 68 & 68 & 52 & 75.09 \\
& & SSA          & 75.00 & 99.2 & 76 & 87 & 67 & 50 & \textbf{75.70} \\
\cmidrule(lr){2-10}

& \multirow{5}{*}{128K}
& Dense             & 92.33 & 98.6 & 95 & 87 & 66 & 50 & 81.49 \\
& & PQ         & 77.67 & 99.2 & 66 & 43 & 66 & 47 & 66.48 \\
& & vAttn (PQ) & 89.33 & 98.0 & 54 & 42 & 54 & 41 & 63.06 \\
& & Oracle-top-$k$          & 88.00 & 98.2 & 79 & 51 & 66 & 49 & \textbf{71.87} \\
& & SSA          & 91.67 & 98.2 & 60 & 67 & 60 & 41 & 69.65 \\

\midrule
\midrule

\multirow{12}{*}{\textbf{20$\times$}}
& \multirow{4}{*}{32K}
& PQ           & 93.33 & 67.6 & 98 & 86 & 79 & 54 & 79.66 \\
& & vAttn (PQ) & 95.33 & 90.8 & 97 & 98 & 76 & 52 & 84.86 \\
& & Oracle-top-$k$          & 96.00 & 73.0 & 99 & 93 & 75 & 52 & 81.33 \\
& & SSA          & 95.67 & 94.2 & 99 & 97 & 79 & 56 & \textbf{86.81} \\
\cmidrule(lr){2-10}

& \multirow{4}{*}{64K}
& PQCache           & 77.00 & 98.4 & 87 & 84 & 70 & 56 & 78.73 \\
& & vAttn (PQ) & 73.00 & 99.8 & 86 & 96 & 73 & 53 & 80.13 \\
& & Oracle          & 77.00 & 99.2 & 93 & 91 & 67 & 52 & 79.87 \\
& & SSA          & 75.33 & 99.8 & 86 & 96 & 73 & 56 & \textbf{81.02} \\
\cmidrule(lr){2-10}

& \multirow{4}{*}{128K}
& PQ           & 88.33 & 99.2 & 71 & 65 & 70 & 50 & 73.92 \\
& & vAttn (PQ) & 89.33 & 98.4 & 59 & 77 & 65 & 44 & 72.12 \\
& & Oracle-top-$k$          & 91.00 & 98.4 & 85 & 75 & 68 & 50 & \textbf{77.90} \\
& & SSA          & 90.33 & 98.6 & 75 & 81 & 67 & 46 & 76.32 \\
\bottomrule
\end{tabular}
\end{table}

\begin{table}[H]
\centering
\caption{Performance on LongBench with Qwen3.5-27B (10K-16K context).}
\label{tab:longbench_qwen35_27b}
\setlength{\tabcolsep}{4pt}

\resizebox{\textwidth}{!}{
\begin{tabular}{cl|cccccccccc}
\toprule
\textbf{Sparsity} & \textbf{Method}
& \textbf{NrtvQA} & \textbf{Qasper} & \textbf{MF-en} & \textbf{MF-zh}
& \textbf{HotpotQA} & \textbf{2WikiMQA} & \textbf{Musique} & \textbf{DuReader}
& \textbf{GovRep} & \textbf{QMSum} \\
\midrule
\multirow{5}{*}{\textbf{50$\times$}}
& Dense               & 20.49 & 29.64 & 38.94 & 42.97 & 34.83 & 34.64 & 26.69 & 25.92 & 31.91 & 21.96 \\
& PQ                  & 20.13 & 28.69 & 38.60 & 42.01 & 33.27 & 31.92 & 20.61 & 25.54 & 31.92 & 21.64 \\
& vAttn(PQ) & 21.97 & 25.33 & 39.38 & 43.53 & 32.00 & 30.44 & 21.45 & 26.91 & 30.68 & 22.42 \\
& SSA                 & 20.42 & 27.51 & 38.87 & 42.13 & 34.64 & 32.79 & 24.88 & 24.82 & 31.96 & 22.72 \\
& Oracle-top-$k$      & 20.51 & 28.97 & 38.22 & 42.12 & 34.79 & 33.77 & 22.50 & 25.76 & 32.86 & 22.52 \\
\midrule
\multirow{4}{*}{\textbf{20$\times$}}
& PQ                  & 20.61 & 29.55 & 38.41 & 41.76 & 34.51 & 33.38 & 23.90 & 25.74 & 31.63 & 22.31 \\
& vAttn(PQ) & 21.55 & 27.17 & 40.14 & 43.67 & 33.39 & 34.18 & 25.65 & 26.70 & 32.20 & 22.27 \\
& SSA                 & 20.90 & 28.80 & 39.73 & 42.91 & 34.41 & 33.97 & 25.24 & 25.93 & 31.36 & 22.58 \\
& Oracle-top-$k$      & 20.74 & 29.28 & 38.68 & 42.87 & 34.90 & 34.57 & 23.46 & 25.82 & 32.77 & 21.97 \\
\bottomrule
\end{tabular}
}

\vspace{6pt}

\resizebox{\textwidth}{!}{
\begin{tabular}{cl|cccccccccc|c}
\toprule
\textbf{Sparsity} & \textbf{Method}
& \textbf{MultiNews} & \textbf{VCSUM} & \textbf{TREC} & \textbf{TriviaQA} & \textbf{SAMSum}
& \textbf{PCount} & \textbf{PRe-en} & \textbf{PRe-zh} & \textbf{LCC} & \textbf{RB-P} & \textbf{Avg.} \\
\midrule
\multirow{5}{*}{\textbf{50$\times$}}
& Dense               & 23.33 & 14.50 & 72.50 & 89.65 & 45.63 & 5.16 & 88.12 & 57.33 & 67.96 & 59.12 & 41.56 \\
& PQ                  & 24.37 & 14.50 & 71.25 & 89.68 & 45.92 & 6.37 & 87.40 & 56.60 & 66.65 & 54.97 & 40.60 \\
& vAttn(PQ) & 23.17 & 14.62 & 71.25 & 89.04 & 43.76 & 6.85 & 85.25 & 59.21 & 66.46 & 53.98 & 40.39 \\
& SSA                 & 23.08 & 14.35 & 72.50 & 89.77 & 44.32 & 5.16 & 86.25 & 59.69 & 64.55 & 58.01 & 40.92 \\
& Oracle-top-$k$      & 24.62 & 14.66 & 72.50 & 89.66 & 46.91 & 5.28 & 87.81 & 53.90 & 68.56 & 59.06 & \textbf{41.25} \\
\midrule
\multirow{4}{*}{\textbf{20$\times$}}
& PQ                  & 24.20 & 14.79 & 70.00 & 89.46 & 45.83 & 5.10 & 89.06 & 54.21 & 67.87 & 57.24 & 40.98 \\
& vAttn(PQ) & 23.41 & 15.05 & 75.00 & 90.52 & 45.12 & 4.68 & 86.88 & 57.23 & 66.02 & 58.56 & 41.47 \\
& SSA                 & 24.01 & 14.67 & 72.50 & 90.18 & 46.93 & 4.76 & 88.12 & 56.81 & 66.85 & 60.47 & \textbf{41.56} \\
& Oracle-top-$k$      & 24.75 & 14.29 & 71.25 & 89.66 & 46.25 & 5.26 & 87.71 & 55.04 & 68.84 & 59.90 & 41.40 \\
\bottomrule
\end{tabular}
}
\end{table}

\begin{table}[H]
\centering
\caption{Performance on RULER-HARD with Qwen3.5-4B at 50$\times$ sparsity.}
\label{tab:ruler_qwen35_4b}
\setlength{\tabcolsep}{5pt}
\begin{tabular}{clccccccc}
\toprule
\textbf{Context} & \textbf{Method}
& \textbf{FWE} & \textbf{VT} & \textbf{NM2} & \textbf{NM3}
& \textbf{QA1} & \textbf{QA2} & \textbf{Avg.} \\
\midrule

\multirow{5}{*}{32K}
& Dense             & 99.67 & 100 & 100 & 100 & 89 & 63 & 91.95 \\
& PQ           & 97.33 & 100 & 100 & 100 & 88 & 62 & 91.22 \\
& vAttn (PQ)   & 96.33 & 100 & 100 & 100 & 89 & 58 & 90.56 \\
& Oracle-top-$k$           & 99.33 & 100 & 100 & 100 & 88 & 61 & 91.39 \\
& SSA            & 99.67 & 100 & 100 & 100 & 89 & 60 & \textbf{91.45} \\
\cmidrule(lr){1-9}

\multirow{5}{*}{64K}
& Dense             & 99.67 & 100 & 100 & 100 & 88 & 59 & 91.11 \\
& PQ           & 98.33 & 100 & 100 & 100 & 88 & 58 & 90.72 \\
& vAttn (PQ)   & 99.67 & 100 & 100 & 100 & 87 & 57 & 90.61 \\
& Oracle-top-$k$            & 99.33 & 100 & 100 & 100 & 86 & 57 & 90.39 \\
& SSA            & 99.00 & 100 & 100 & 100 & 88 & 58 & \textbf{90.83} \\
\cmidrule(lr){1-9}

\multirow{5}{*}{128K}
& Dense             & 96.67 & 100 & 100 & 100 & 84 & 56 & 89.45 \\
& PQ           & 95.33 & 100 & 100 & 100 & 82 & 55 & 88.72 \\
& vAttn (PQ)   & 96.67 & 100 & 100 & 100 & 82 & 57 & 89.28 \\
& Oracle-top-$k$            & 97.33 & 100 & 100 & 100 & 82 & 57 & 89.39 \\
& SSA            & 97.00 & 100 & 100 & 100 & 82 & 58 & \textbf{89.50} \\

\bottomrule
\end{tabular}
\end{table}

\begin{table}[H]
\centering
\caption{Performance on RULER-HARD with Qwen3.5-9B at 50$\times$ sparsity.}
\label{tab:ruler_hard_qwen35_9b}
\setlength{\tabcolsep}{5pt}
\begin{tabular}{clccccccc}
\toprule
\textbf{Context} & \textbf{Method}
& \textbf{FWE} & \textbf{VT} & \textbf{NM2} & \textbf{NM3}
& \textbf{QA1} & \textbf{QA2} & \textbf{Avg.} \\
\midrule

\multirow{5}{*}{32K}
& Dense             & 99.67 & 100 & 100 & 100 & 93 & 68 & 93.45 \\
& PQ         & 94.00 & 100 & 100 & 100 & 91 & 67 & 92.00 \\
& vAttn (PQ)   & 98.00 & 100 & 100 & 100 & 89 & 63 & 91.67 \\
& Oracle-top-$k$            & 98.33 & 100 & 99 & 100 & 91 & 67 & 92.56 \\
& SSA            & 100.00 & 100 & 100 & 100 & 93 & 67 & \textbf{93.33} \\
\cmidrule(lr){1-9}

\multirow{5}{*}{64K}
& Dense             & 99.33 & 100 & 100 & 100 & 88 & 66 & 92.22 \\
& PQ           & 95.67 & 99.8 & 100 & 100 & 86 & 65 & 91.08 \\
& vAttn (PQ)   & 98.00 & 99.8 & 100 & 99 & 88 & 63 & 91.30 \\
& Oracle-top-$k$            & 98.67 & 100 & 100 & 100 & 89 & 66 & 92.28 \\
& SSA            & 99.33 & 99.8 & 100 & 100 & 89 & 66 & \textbf{92.36} \\
\cmidrule(lr){1-9}

\multirow{5}{*}{128K}
& Dense             & 99.67 & 100 & 100 & 100 & 87 & 61 & 91.28 \\
& PQ         & 97.67 & 100 & 98 & 100 & 88 & 62 & 90.95 \\
& vAttn (PQ)   & 98.67 & 100 & 98 & 100 & 84 & 63 & 90.61 \\
& Oracle-top-$k$            & 99.33 & 100 & 99 & 100 & 88 & 62 & 91.39 \\
& SSA            & 99.67 & 100 & 100 & 100 & 88 & 61 & \textbf{91.45} \\
\bottomrule
\end{tabular}
\end{table}

\begin{table}[H]
\centering
\caption{
Performance on RULER-HARD with Ministral-3-8B.
}
\label{tab:ministral_ruler_hard}
\resizebox{\textwidth}{!}{
\begin{tabular}{cc|l|cccccc|c}
\toprule
\textbf{Context} & \textbf{Sparsity} & \textbf{Method}
& \textbf{FWE} & \textbf{VT} & \textbf{NM2} & \textbf{NM3}
& \textbf{QA1} & \textbf{QA2} & \textbf{Avg.} \\
\midrule

\multirow{5}{*}{32K}
& \multirow{5}{*}{50$\times$}
& Dense
& 94.00 & 100.00 & 99.00 & 100.00 & 80.00 & 63.00 & 89.33 \\

& & PQ
& 88.00 & 100.00 & 97.00 & 67.00 & 80.00 & 59.00 & 81.83 \\

& & vAttn (PQ)
& 89.33 & 98.80 & 85.00 & 87.00 & 64.00 & 48.00 & 78.69 \\

& & SSA
& \textbf{94.00} & \textbf{100.00} & 88.00 & \textbf{91.00}
& 74.00 & 58.00 & 84.17 \\

& & Oracle-top-$k$
& 90.67 & \textbf{100.00} & \textbf{99.00} & 90.00
& \textbf{81.00} & \textbf{59.00} & \textbf{86.61} \\

\midrule

\multirow{5}{*}{64K}
& \multirow{5}{*}{50$\times$}
& Dense
& 81.00 & 99.80 & 100.00 & 100.00 & 75.00 & 61.00 & 86.13 \\

& & PQ
& 84.33 & 99.80 & 89.00 & 71.00 & 68.00 & 56.00 & 78.02 \\

& & vAttn (PQ)
& 81.33 & 99.80 & 71.00 & 87.00 & 63.00 & 49.00 & 75.19 \\

& & SSA
& 81.67 & 99.80 & 80.00 & \textbf{95.00}
& 69.00 & 54.00 & 79.91 \\

& & Oracle-top-$k$
& \textbf{84.67} & 99.80 & \textbf{100.00} & 82.00
& \textbf{75.00} & \textbf{59.00} & \textbf{83.41} \\

\midrule

\multirow{5}{*}{128K}
& \multirow{5}{*}{50$\times$}
& Dense
& 92.67 & 100.00 & 100.00 & 90.00 & 62.00 & 44.00 & 81.45 \\

& & PQ
& 79.33 & 99.60 & 63.00 & 32.00 & 63.00 & 41.00 & 62.99 \\

& & vAttn (PQ)
& 89.67 & \textbf{100.00} & 36.00 & 61.00 & 59.00 & 37.00 & 63.78 \\

& & SSA
& \textbf{90.67} & 99.80 & 50.00 & \textbf{71.00}
& 62.00 & 42.00 & 69.25 \\

& & Oracle-top-$k$
& 85.00 & \textbf{100.00} & \textbf{93.00} & 67.00
& \textbf{64.00} & \textbf{46.00} & \textbf{75.83} \\

\midrule\midrule

\multirow{4}{*}{32K}
& \multirow{4}{*}{20$\times$}
& PQ
& 92.33 & 100.00 & \textbf{99.00} & 97.00
& \textbf{83.00} & 60.00 & \textbf{88.56} \\

& & vAttn (PQ)
& \textbf{95.00} & 100.00 & 96.00 & \textbf{100.00}
& 78.00 & 61.00 & 88.33 \\

& & SSA
& 93.33 & 100.00 & 97.00 & 99.00
& 77.00 & \textbf{63.00} & 88.22 \\

& & Oracle-top-$k$
& 92.00 & 100.00 & \textbf{99.00} & 98.00
& 78.00 & \textbf{63.00} & 88.33 \\

\midrule

\multirow{4}{*}{64K}
& \multirow{4}{*}{20$\times$}
& PQ
& \textbf{83.67} & 99.80 & 99.00 & 97.00
& \textbf{74.00} & \textbf{60.00} & \textbf{85.58} \\

& & vAttn (PQ)
& 83.33 & 99.80 & 94.00 & 99.00
& 72.00 & 57.00 & 84.19 \\

& & SSA
& 81.67 & 99.80 & 95.00 & \textbf{100.00}
& 73.00 & 58.00 & 84.58 \\

& & Oracle-top-$k$
& 83.00 & 99.80 & \textbf{100.00} & 95.00
& 71.00 & \textbf{60.00} & 84.80 \\

\midrule

\multirow{4}{*}{128K}
& \multirow{4}{*}{20$\times$}
& PQ
& 85.33 & \textbf{100.00} & 88.00 & 80.00
& \textbf{66.00} & 44.00 & 77.22 \\

& & vAttn (PQ)
& 91.67 & 99.80 & 77.00 & \textbf{89.00}
& 62.00 & 43.00 & 77.08 \\

& & SSA
& \textbf{93.33} & 99.80 & 81.00 & \textbf{89.00}
& 62.00 & 42.00 & 77.86 \\

& & Oracle-top-$k$
& 90.67 & \textbf{100.00} & \textbf{99.00} & 84.00
& 64.00 & \textbf{44.00} & \textbf{80.28} \\

\bottomrule
\end{tabular}
}
\end{table}

\begin{table}[H]
\centering
\caption{
Ablation on the heavy-hitter fraction $h$ for SSA on AIME 2025.
A fraction $h$ of the sparse attention budget is allocated to deterministic
heavy hitters, while the remaining $1-h$ is allocated to importance sampling.
Thus, $h=0$ corresponds to pure importance sampling and $h=1$ corresponds
to PQCache (pure top-$k$ retrieval).
}
\label{tab:aime_heavy_size}
\begin{tabular}{c|cc}
\toprule
\textbf{Heavy Fraction $h$} & \textbf{20$\times$} & \textbf{10$\times$} \\
\midrule
0.0 & 40.2 & 42.5 \\
0.2 & 61.0 & 63.8 \\
0.4 & 63.5 & \textbf{64.0} \\
0.6 & 62.7 & 62.7 \\
0.8 & \textbf{63.8} & 63.1 \\
1.0 (PQ) & 49.4 & 53.8 \\
\bottomrule
\end{tabular}
\end{table}

\end{document}